\documentclass{amsart}

\usepackage{amsmath,amssymb,amsthm}
\usepackage{graphicx}
\usepackage{tikz}
\usepackage{pgfplots}
\pgfplotsset{compat=1.16}
\usepgfplotslibrary{groupplots}
\usetikzlibrary{calc}
\usepackage[shortlabels]{enumitem}

\makeatletter

\@ifundefined{toprule}{\usepackage{booktabs}}{}
\providecommand\botrule{\bottomrule}

\providecommand\bmhead[1]{\subsection*{#1}}
\@ifundefined{th@thmstyleone}{\expandafter\let\csname th@thmstyleone\endcsname\th@plain}{}
\@ifundefined{th@thmstylethree}{\expandafter\let\csname th@thmstylethree\endcsname\th@definition}{}

\makeatother

\usepackage{xcolor}
\makeatletter
\@ifpackageloaded{hyperref}{}{\usepackage{hyperref}}
\makeatother
\hypersetup{colorlinks=true,linkcolor=blue,citecolor=blue,urlcolor=blue}

\usepackage[capitalise,nameinlink]{cleveref}

\theoremstyle{thmstyleone}
\newtheorem{thm}{Theorem}[section]
\newtheorem{lem}[thm]{Lemma}
\newtheorem{prop}[thm]{Proposition}
\newtheorem{cor}[thm]{Corollary}
\theoremstyle{thmstylethree}
\newtheorem{defn}[thm]{Definition}
\newtheorem{exa}[thm]{Example}
\newtheorem{rem}[thm]{Remark}

\crefname{thm}{Thm.}{Thms.}
\crefname{lem}{Lem.}{Lems.}
\crefname{prop}{Prop.}{Props.}
\crefname{cor}{Cor.}{Cors.}
\crefname{defn}{Def.}{Defs.}
\crefname{exa}{Ex.}{Exs.}
\crefname{rem}{Rem.}{Rems.}
\crefname{section}{Sec.}{Secs.}
\crefname{subsection}{Sec.}{Secs.}
\crefname{subsubsection}{Sec.}{Secs.}
\crefname{figure}{Fig.}{Figs.}
\crefname{table}{Tab.}{Tabs.}
\Crefname{thm}{Theorem}{Theorems}
\Crefname{lem}{Lemma}{Lemmas}
\Crefname{prop}{Proposition}{Propositions}
\Crefname{cor}{Corollary}{Corollaries}
\Crefname{defn}{Definition}{Definitions}
\Crefname{exa}{Example}{Examples}
\Crefname{rem}{Remark}{Remarks}
\Crefname{section}{Section}{Sections}
\Crefname{subsection}{Section}{Sections}
\Crefname{subsubsection}{Section}{Sections}
\Crefname{figure}{Figure}{Figures}
\Crefname{table}{Table}{Tables}
\crefname{equation}{}{}
\Crefname{equation}{Eq.}{Eqs.}
\crefformat{equation}{(#2#1#3)}
\crefrangeformat{equation}{(#3#1#4)--(#5#2#6)}
\crefmultiformat{equation}{(#2#1#3)}{ and (#2#1#3)}{, (#2#1#3)}{ and (#2#1#3)}

\providecommand\R{\mathbb R}
\providecommand\C{\mathbb C}
\providecommand\Z{\mathbb Z}
\providecommand\PP{\mathbb P}
\providecommand\Sph{\mathbb S^2}
\providecommand\Ball{\mathbb B^3}
\providecommand\Harm{\mathcal H}
\providecommand\Sym{\mathrm{Sym}}
\providecommand\Mob{\mathrm{M\ddot ob}}
\providecommand\Or{\mathrm O}
\providecommand\SO{\mathrm{SO}}
\providecommand\SU{\mathrm{SU}}
\providecommand\SL{\mathrm{SL}}
\providecommand\PSL{\mathrm{PSL}}
\providecommand\GL{\mathrm{GL}}
\DeclareMathOperator{\Fix}{Fix}
\DeclareMathOperator{\Stab}{Stab}

\DeclareMathOperator{\Gram}{Gram}

\DeclareMathOperator{\diag}{diag}
\DeclareMathOperator{\disc}{disc}
\DeclareMathOperator{\tr}{tr}
\DeclareMathOperator{\sgn}{sgn}
\providecommand\abs[1]{\lvert#1\rvert}
\providecommand\norm[1]{\lVert#1\rVert}
\providecommand\ip[2]{\langle#1,#2\rangle}
\providecommand\dw{\,d\omega}
\begin{document}

\title[Chirality of spherical harmonics]{Chirality of spherical harmonics and automorphisms of hyperelliptic curves}

\author{Tanush Shaska}
\address{Department of Computer Science and Engineering, Oakland University, Rochester, MI 48309, USA}
\email{shaska@oakland.edu}

\begin{abstract}
Every real spherical harmonic of even degree is congruent to its mirror image.  We prove that a real spherical harmonic of odd degree $n\ge3$ with distinct Maxwell axes is congruent to its mirror image if and only if the hyperelliptic curve of genus $n-1$ defined by its binary form has a reduced automorphism of even order whose lift acts on the form by $-1$.  We then determine the locus of these curves: it is the locus of curves with an extra involution if and only if $n-1$ is a power of two, and otherwise it also contains, for each even divisor $k$ of $n-1$ with odd cofactor $d\ge3$, an irreducible family $y^2=xK(x^{2k})$, $\deg K=d$, with generic reduced automorphism group $\Z/2k$.  For $n=3$ the condition is the vanishing of the skew invariant of the binary sextic.  As an application, the invariants of the truncated spherical harmonic expansion of a genus-zero surface, in a conformal parameterization normalized by its conformal barycenter, give a descriptor that is exactly invariant under parameterization, position and scale, records handedness, and is stable in the orbit distance.
\end{abstract}



\maketitle

\section{Introduction}
\label{sec:intro}

A real spherical harmonic $f$ of degree $n$ is, by a theorem of Maxwell and Sylvester, the harmonic part of a product of $n$ real linear forms, and the $n$ lines normal to the planes on which these forms vanish, the \emph{Maxwell axes} of $f$, are determined by $f$ up to order.  We ask when $f$ is congruent to its mirror image, that is, when $f\circ S$ lies in the $\SO(3)$-orbit of $f$ for a reflection $S$.  If $n$ is even the answer is always, since $-1\in\Or(3)$ acts trivially on harmonics of even degree.  If $n$ is odd the question is substantial, and this paper answers it.

The answer uses the classical identification of a harmonic of degree $n$ with a binary form of degree $2n$, the restriction of the harmonic polynomial to the null cone $x_1^2+x_2^2+x_3^2=0$ (see~\cite{Cartan1938,Olive14}).  Rotations act through $\SU(2)\subset\SL_2(\C)$, so rotation invariants of harmonics are invariants of binary forms, and the real harmonics are the fixed points of an antilinear involution of the forms which is not conjugation of the coefficients.  This real structure sorts the invariants: an invariant of degree $d$ is real on real harmonics when $nd$ is even, and purely imaginary, with opposite signs on $f$ and its mirror image, when $nd$ is odd (\cref{lem:reality}).  The invariants of the second kind, the pseudo-invariants, detect chirality: $f$ is achiral iff all of them vanish (\cref{prop:chirality}).  For odd $n$ they are the invariants of odd degree, and when the axes are distinct the binary form of $f$ defines a hyperelliptic curve of genus $n-1$.  
Let $n$ be odd and let $F$ be the binary form of a nonzero real harmonic $f$ of degree $n$.  We prove (\cref{thm:chiral-odd,thm:hyperelliptic}) that $f$ is congruent to its mirror image iff every invariant of odd degree of binary forms of degree $2n$ vanishes at $F$, iff $g\cdot F=-F$ for some $g\in\SL_2(\C)$, iff some rotation of even order $2k$ permutes the Maxwell axes of $f$ in such a way that the axes orthogonal to its axis form, counted with multiplicity, an odd number of orbits; and that, when $n\ge3$ and the axes are distinct, these hold iff the hyperelliptic curve $y^2=F(x,1)$ is isomorphic to $y^2=x^rK(x^{2k})$ for some $k\ge1$, with $r\in\{0,1\}$ and $r\equiv n-k\pmod{2k}$, that is, iff the curve has a reduced automorphism of order $2k$ whose lift to $\SL_2(\C)$ acts on $F$ by $-1$.

The rotation can be taken to be a half-turn iff the curve has an extra involution, that is, lies in the locus $\mathcal L_{n-1}$ of~\cite{2003-3} (\cref{cor:Ln}).  The locus $\mathcal A_n$ of curves with such a reduced automorphism, which contains the curves of the achiral harmonics with distinct axes, is the union of $\mathcal L_{n-1}$ with the loci of curves $y^2=xK(x^{2k})$, one for each even divisor $k$ of $n-1$ with odd cofactor $(n-1)/k=\deg K$, each irreducible of dimension $(n-1)/k-1$ and, when $\deg K\ge3$, with generic reduced automorphism group $\Z/2k$ (\cref{thm:strata}); it coincides with $\mathcal L_{n-1}$ iff $n-1$ is a power of two.  The curves of $\mathcal A_n$ that arise from real harmonics are those whose branch locus is invariant under a fixed-point-free antiholomorphic involution (\cref{thm:real}).

Degree three is the test case, where the curve has genus two and the invariant theory is classical (\cref{thm:chirality3}): the odd invariants of the sextic are the multiples of the skew invariant $R$, the locus $\mathcal L_2$ has an explicit equation $F_2=0$ in the Igusa invariants~\cite{2000-2}, and $R^2$ is a constant multiple of $F_2$.

The question comes from shape analysis, where a closed surface of genus zero is described by the spherical harmonic coefficients of a function on it, and where a descriptor built from reflection-invariant quantities, such as the power spectrum (see~\cite{KazhdanFunkhouserRusinkiewicz03}) or the bispectrum through degree three (see~\cite{Kakarala12}), cannot distinguish the surface from its mirror image.  The coefficients depend on the parameterization of the surface by the sphere, on its position and on its scale.  A conformal parameterization (see~\cite{GuWangChanThompsonYau04,ChoiLamLui15}) normalized by its conformal barycenter (see~\cite{DouadyEarle86,BadenCraneKazhdan18}) reduces the first dependence to one rotation of the sphere (\cref{cor:orbit}); position and scale are removed by a canonical center and by weighted projectivization; and the invariants of the truncated radial expansion then form a descriptor of the surface that is exactly invariant, complete for the truncation, records handedness, and is stable in the orbit distance (\cref{thm:certified}).  Through degree two the descriptor is the integrity basis of a vector and a symmetric tensor (see~\cite{SpencerRivlin62}) completed by one pseudo-invariant (\cref{thm:L2}), and the stability exponent is at most $1/2$ near points with a nonprincipal stabilizer and at most $1/6$ near the origin, the latter because of chirality (\cref{thm:stability,exa:holder}).  The recovery of orbits of truncated spherical harmonics from invariants is studied in~\cite{BandeiraEtAl23,EdidinSatriano24,BendoryEtAl25} for generic orbits and invariants of degree three; the present paper treats all orbits in low degree and the behaviour under reflection, which generic recovery does not see.  Applications of this material are given in~\cite{ShaskaSiadat26}.

\section{Spherical Harmonics and Binary Forms}
\label{sec:binary}

This section identifies the real spherical harmonics of degree $l$ with a real form of the binary forms of degree $2l$.  Rotations act through the double cover $\SU(2)\subset\SL_2(\C)$ of $\SO(3)$, and rotation invariants of harmonics become classical invariants of binary forms.  The parameterization of the null conic and the spin covering go back to Cartan~\cite{Cartan1938}, $\Harm_l^\C\cong\Sym^{2l}$ is the standard realization of the irreducible representations of $\SU(2)$ (see~\cite{Weyl46,Vilenkin68}), and the factorization of a harmonic into its Maxwell axes is the theorem of Maxwell and Sylvester~\cite{Maxwell1873,Sylvester1876}; these facts are stated without proof.  What is needed beyond them is the real structure of \cref{lem:reality} and the normalizations of \cref{thm:correspondence}.

\subsection{Spherical harmonics}
\label{sec:harmonics}

Let $\Sph=\{u\in\R^3:\abs u=1\}$ be the unit sphere, $\ip\cdot\cdot$ the Euclidean inner product on $\R^3$ and its complex bilinear extension to $\C^3$, $\abs\cdot$ the Euclidean norm, and $d\omega=dA_{\rm round}/4\pi$ the round probability measure on $\Sph$.  We identify $\Sph$ with $\PP^1=\{[z_0:z_1]\}$ through the coordinate
\[
\zeta=\frac{z_0}{z_1}=\frac{-n_1+in_2}{1-n_3},\qquad n\in\Sph;
\]
this is stereographic projection from the north pole followed by a reflection, it preserves the outward orientation of $\Sph$, and it is the identification compatible with the Veronese map below.  Every element of $\Or(3)\setminus\SO(3)$ has the form $-R$ with $R\in\SO(3)$.

For $l\ge0$ let $\Harm_l$ be the space of restrictions to $\Sph$ of homogeneous harmonic polynomials of degree $l$ on $\R^3$; $\dim\Harm_l=2l+1$, $-\Delta_{\Sph}$ acts on $\Harm_l$ by $l(l+1)$, and $L^2(\Sph,d\omega)$ is the orthogonal Hilbert sum of the $\Harm_l$, see~\cite{AxlerBourdonRamey01}.  For instance $\Harm_1=\{\ip a\cdot:a\in\R^3\}$ and $\Harm_2=\{u\mapsto u^{\top}Qu:Q=Q^{\top},\ \tr Q=0\}$.  Write $u=(\sin\theta\cos\phi,\sin\theta\sin\phi,\cos\theta)$, let $P_l$ be the Legendre polynomial and $P_l^m(t)=(1-t^2)^{m/2}P_l^{(m)}(t)$.  We fix the orthonormal basis
\begin{equation}
\label{eq:basis}
Y_{l0}=\sqrt{2l+1}\,P_l(\cos\theta),\qquad
\begin{aligned}
Y_{lm}&=N_{lm}\,P_l^m(\cos\theta)\cos m\phi,\\
Y_{l,-m}&=N_{lm}\,P_l^m(\cos\theta)\sin m\phi,
\end{aligned}
\qquad 1\le m\le l,
\end{equation}
with $N_{lm}=\sqrt{2(2l+1)(l-m)!/(l+m)!}$.  This fixes normalization and signs; software libraries differ from it by constant factors such as the Condon--Shortley factor $(-1)^m$.  The \textbf{coefficients} of $f\in L^2(\Sph,d\omega)$ are
\[
c_{lm}(f)=\int_{\Sph}f\,Y_{lm}\dw,
\]
the orthogonal projection onto $\Harm_l$ is $\pi_lf=\sum_{\abs m\le l}c_{lm}(f)Y_{lm}$, and $l$ is the \textbf{band}.  For $\R^3$-valued functions $\pi_l$ is taken componentwise.

The group $\SO(3)$ acts by $(Q\cdot f)(u)=f(Q^{-1}u)$.  It preserves each $\Harm_l$ and acts on $c_l(f)=(c_{l,-l},\dots,c_{ll})$ by an orthogonal matrix $D_l(Q)$, the real form of the Wigner matrix, see~\cite{Vilenkin68}.  The band norms $\norm{\pi_lf}$ are invariant; the individual coefficients are not.

\subsection{Binary forms and the null cone}
Let $P_l(\C^3)$ be the homogeneous polynomials of degree $l$ in $x=(x_1,x_2,x_3)$ with complex coefficients and $\Harm_l^\C=\ker(\Delta\colon P_l(\C^3)\to P_{l-2}(\C^3))$, the complexification of $\Harm_l$; restriction to $\Sph$ identifies real harmonic polynomials with $\Harm_l$, and we use the same letter for both.  \textbf{Complex conjugation} on $\Harm_l^\C$ is $\bar f(x)=\overline{f(\bar x)}$, and its fixed points are the real harmonics.  Let $q(x)=x_1^2+x_2^2+x_3^2=\ip xx$, with $\ip\cdot\cdot$ the complex bilinear pairing; for $x\in\C^3$, $q(x)$ can vanish with $x\ne0$.  By the Fischer decomposition (see~\cite{Fischer1918,AxlerBourdonRamey01}), $P_l(\C^3)=\Harm_l^\C\oplus q\,P_{l-2}(\C^3)$ for $l\ge2$, $\SO_3(\C)$-equivariantly and also over $\R$, so every $P\in P_l(\C^3)$ has a unique \textbf{harmonic part} $\pi(P)\in\Harm_l^\C$ with $P-\pi(P)\in q\,P_{l-2}(\C^3)$, and $\pi$ is equivariant.

Let $\Sym^n$ be the binary forms $F(z_0,z_1)$ of degree $n$, with $\SL_2(\C)$ acting by $(g\cdot F)(z)=F(g^{-1}z)$.  The \textbf{roots} of $F\ne0$ are the points $[z]\in\PP^1$ with $F(z)=0$; with multiplicity they form the \textbf{root divisor}, which determines $F$ up to a scalar.  For $F\in\Sym^m$, $G\in\Sym^n$ and $0\le k\le\min(m,n)$, the $k$-th \textbf{transvectant} is
\begin{equation}
\label{eq:transvectant}
(F,G)_k=\frac{(m-k)!\,(n-k)!}{m!\,n!}\sum_{i=0}^k(-1)^i\binom ki\,\frac{\partial^kF}{\partial z_0^{k-i}\partial z_1^i}\,\frac{\partial^kG}{\partial z_0^i\partial z_1^{k-i}}.
\end{equation}
It is bilinear, $\SL_2(\C)$-equivariant, with rational coefficients, of degree $m+n-2k$, and $(F,F)_k=0$ for odd $k$; for $k=m=n$ it is an invariant.  We write forms in binomial normalization $F=\sum_k\binom{2l}ka_kz_0^{2l-k}z_1^k$.  The \textbf{null cone} is $N=\{x\in\C^3:q(x)=0\}$; every polynomial vanishing on $N$ is divisible by $q$.  The \textbf{Veronese map}
\begin{equation}
\label{eq:veronese}
v(z_0,z_1)=\bigl(z_0^2-z_1^2,\; i(z_0^2+z_1^2),\; 2z_0z_1\bigr)
\end{equation}
maps $\C^2$ onto $N$, with $v(z)=v(z')$ iff $z'=\pm z$, and the \textbf{quadratic} of $x\in\C^3$ is
\begin{equation}
\label{eq:Qx}
Q_x(z)=\ip x{v(z)}=(x_1+ix_2)z_0^2+2x_3z_0z_1+(-x_1+ix_2)z_1^2 .
\end{equation}
The map $x\mapsto Q_x$ is an isomorphism $\C^3\to\Sym^2$, and $\disc(Q_x)=4q(x)$ in the normalization $\disc(\alpha z_0^2+2\beta z_0z_1+\gamma z_1^2)=4(\beta^2-\alpha\gamma)$.  For $z\ne0$ the vector $m(z)=i\,v(z)\times\overline{v(z)}/(2\abs z^4)$, $\abs z^2=\abs{z_0}^2+\abs{z_1}^2$, is real of unit length and depends only on $[z]$; the induced map $\bar m\colon\PP^1\to\Sph$ is the Hopf map and coincides with the identification of \cref{sec:harmonics}.  With $j(z_0,z_1)=(-\bar z_1,\bar z_0)$, two substitutions are used below:
\begin{equation}
\label{eq:hopf}
Q_a(z)=0\iff m(z)=\pm a/\abs a\quad(a\in\R^3\setminus0),\qquad v(jz)=-\overline{v(z)},
\end{equation}
and $[z]\mapsto[jz]$ is the antipodal map of $\Sph$.  Thus the root divisor of $Q_a$, for real $a\ne0$, is the pair of points of the axis $\R a$ on the sphere.

For $g\in\SL_2(\C)$ let $\rho(g)\in\GL_3(\C)$ be defined by $v(gz)=\rho(g)\,v(z)$.  This is the spin covering: $\rho$ is a surjective homomorphism $\SL_2(\C)\to\SO_3(\C)$ with kernel $\{\pm1\}$, it restricts to a surjective homomorphism $\SU(2)\to\SO(3)$, $g\cdot Q_x=Q_{\rho(g)x}$, and $m(gz)=\rho(g)\,m(z)$ for $g\in\SU(2)$, see~\cite{Cartan1938,Vilenkin68}.  For $g_\theta=\diag(e^{i\theta/2},e^{-i\theta/2})$ one has $\rho(g_\theta)=R_{-\theta}$, the rotation by $-\theta$ about the $x_3$-axis.

\begin{defn}
\label{def:psi}
For $f\in\Harm_l^\C$ the \textbf{binary form of $f$} is $F_f=f\circ v\in\Sym^{2l}$, and $\Psi_l\colon\Harm_l^\C\to\Sym^{2l}$ is the linear map $\Psi_l(f)=F_f$.
\end{defn}

On linear forms~\cref{eq:veronese} gives $\Psi_1(x_1+ix_2)=-2z_1^2$, $\Psi_1(x_1-ix_2)=2z_0^2$ and $\Psi_1(x_3)=2z_0z_1$.  Since evaluation at $v(z)$ is multiplicative and kills $q\,P_{l-2}(\C^3)$, $\Psi_l\bigl(\pi(P)\bigr)(z)=P(v(z))$ for every $P\in P_l(\C^3)$.

\subsection{The real form and reflection parity}

A \textbf{real structure} on a complex vector space $W$ is an antilinear map $\sigma$ with $\sigma^2=1$; its fixed set $\Fix(\sigma)$ is the \textbf{real form} of $W$.  Complex conjugation is a real structure on $\Harm_l^\C$ with real form $\Harm_l$.  We transport it to $\Sym^{2l}$.  For $F\in\Sym^{2l}$ define
\begin{equation}
\label{eq:sigma}
\sigma_l(F)(z)=(-1)^l\,\overline{F(jz)}.
\end{equation}
If $F=\sum_{k=0}^{2l}\binom{2l}ka_kz_0^{2l-k}z_1^k$, then
\begin{equation}
\label{eq:sigma-coeff}
\sigma_l(F)=\sum_{k=0}^{2l}\binom{2l}k(-1)^{l+k}\,\bar a_{2l-k}\,z_0^{2l-k}z_1^k .
\end{equation}
Thus $F\in\Fix(\sigma_l)$ iff $a_{2l-k}=(-1)^{l+k}\bar a_k$ for all $k$; this is not the condition that the coefficients be real.  On forms of odd degree the same formula squares to $-1$ and is a quaternionic structure, which is why real harmonics correspond to forms of even degree.

Let $V=\bigoplus_{l\le L}\Sym^{2l}$ with $\SL_2(\C)$ acting summandwise.  An \textbf{invariant} is a polynomial $I\colon V\to\C$ with $I(g\cdot F)=I(F)$.  It is \textbf{multihomogeneous of multidegree $(d_l)$} if $I((t_lF_l)_l)=\prod_lt_l^{d_l}I((F_l)_l)$; every invariant is a sum of multihomogeneous ones.  It \textbf{has real coefficients} if it is a real polynomial in the binomial coefficients $a_k$ of the forms.

\begin{lem}
\label{lem:reality}
The following are true:
\begin{enumerate}[label=(\roman*)]
\item $\sigma_l$ is a real structure on $\Sym^{2l}$ commuting with $\SU(2)$, and $\Psi_l(\bar f)=\sigma_l(\Psi_lf)$ for all $f\in\Harm_l^\C$.  Hence $\Psi_l(\Harm_l)=\Fix(\sigma_l)$.
\item A nonzero $F\in\Sym^{2l}$ is a complex multiple of a $\sigma_l$-fixed form if and only if its root divisor is invariant under the antipodal map.  Every such $F$ factors as $F=\lambda\,Q_{a_1}\cdots Q_{a_l}$ with $a_k\in\R^3\setminus\{0\}$ and $\lambda\in\C^\times$; the lines $\R a_k$ are determined by $F$ up to order, with multiplicity; and $\lambda$ is real if and only if $F$ is $\sigma_l$-fixed.
\item Let $I$ be a multihomogeneous invariant of multidegree $(d_l)$ with real coefficients.  Then $I\bigl((F_{f_l})_l\bigr)=(-1)^{\sum_l l\,d_l}\;\overline{I\bigl((F_{f_l})_l\bigr)}$ for all real harmonics $f_l\in\Harm_l$, $l\le L$.
\item The map $-1\in\Or(3)$ acts on $\Harm_l$ by $(-1)^l$, and it multiplies $I$ by $(-1)^{\sum_l l d_l}$.  If $\sum_lld_l$ is even, $I$ is invariant under $\Or(3)$; if it is odd, $I$ changes sign under every reflection.  Every invariant decomposes uniquely as the sum of an even and an odd part.
\end{enumerate}
\end{lem}

\begin{proof}
(i) Since $j^2=-1$ and $F$ has even degree, $\sigma_l(\sigma_lF)(z)=(-1)^l\overline{\sigma_l(F)(jz)}=F(j^2z)=F(-z)=F(z)$, so $\sigma_l^2=1$, and $\sigma_l$ is a real structure.  For $g\in\SU(2)$, $jg^{-1}=g^{-1}j$, so $\sigma_l(g\cdot F)(z)=(-1)^l\overline{F(g^{-1}jz)}=(-1)^l\overline{F(jg^{-1}z)}=(g\cdot\sigma_lF)(z)$.  For $f\in\Harm_l^\C$, using $\bar f(x)=\overline{f(\bar x)}$ and~\cref{eq:hopf}, $\Psi_l(\bar f)(z)=\overline{f\bigl(\overline{v(z)}\bigr)}=\overline{f\bigl(-v(jz)\bigr)}=(-1)^l\,\overline{(\Psi_lf)(jz)}=\sigma_l(\Psi_lf)(z)$, using that $f$ is homogeneous of degree $l$.  Thus $\Psi_l$ intertwines conjugation and $\sigma_l$.  A linear isomorphism that intertwines two real structures maps the fixed set of one onto the fixed set of the other, so $\Psi_l(\Harm_l)=\Fix(\sigma_l)$.

(ii) By~\cref{eq:sigma}, $\sigma_l(F)(z)=0$ if and only if $F(jz)=0$, so the root divisor of $\sigma_lF$ is the image of the root divisor of $F$ under the antipodal map $[z]\mapsto[jz]$, with the same multiplicities.  Suppose the root divisor of $F$ is antipodally invariant.  Then $\sigma_lF$ and $F$ have the same root divisor, so $\sigma_lF=cF$ for some $c\in\C^\times$.  Applying $\sigma_l$ again, $F=\sigma_l(cF)=\bar c\,\sigma_lF=\bar cc\,F$, so $\abs c=1$.  Write $c=e^{2i\theta}$.  Then $\sigma_l(e^{i\theta}F)=e^{-i\theta}cF=e^{i\theta}F$, so $e^{i\theta}F$ is $\sigma_l$-fixed.  Conversely, if $F=\mu G$ with $G$ fixed, the root divisor of $F$ equals that of $G$, which is antipodally invariant.

Now let $F$ be $\sigma_l$-fixed and nonzero.  The antipodal map has no fixed points on $\PP^1$, so the $2l$ roots of $F$, counted with multiplicity, fall into $l$ antipodal pairs.  Let $\{p,-p\}$ be one of them, with $p\in\Sph$, and let $a\in\R^3\setminus\{0\}$ be any vector with $a/\abs a=\pm p$.  By~\cref{eq:hopf}, $Q_a$ has root divisor $\{p,-p\}$.  A binary form is determined by its root divisor up to a scalar, so $F=\lambda\,Q_{a_1}\cdots Q_{a_l}$ with $a_k$ chosen this way for the $l$ pairs, and $\lambda\in\C^\times$.  The lines $\R a_k$ are determined by the root pairs, hence by $F$, up to order and with multiplicity.  Each $Q_{a_k}$ is $\sigma_1$-fixed by~\cref{eq:sigma-coeff}, and $\sigma_l(FG)=\sigma_{l_1}(F)\sigma_{l_2}(G)$ for $F\in\Sym^{2l_1}$, $G\in\Sym^{2l_2}$, $l=l_1+l_2$, directly from~\cref{eq:sigma}.  Hence $\prod_kQ_{a_k}$ is $\sigma_l$-fixed, and $\sigma_l(F)=\bar\lambda\prod_kQ_{a_k}$.  So $F$ is fixed if and only if $\bar\lambda=\lambda$.

(iii) Let $w\in\SL_2(\C)$ be the matrix with $wz=(-z_1,z_0)$.  For $F=\sum_k\binom{2l}ka_kz_0^{2l-k}z_1^k$,
\[
(w^{-1}\cdot F)(z)=F(wz)=\sum_k\binom{2l}ka_k(-z_1)^{2l-k}z_0^k=\sum_k\binom{2l}k(-1)^{k}a_{2l-k}\,z_0^{2l-k}z_1^k,
\]
using $(-1)^{2l-k}=(-1)^k$ and reindexing.  Comparing with~\cref{eq:sigma-coeff}, $\sigma_l(F)=\overline{(-1)^l\,w^{-1}\cdot F}$, where the bar is conjugation of all coefficients.  Now let $I$ be as in the statement and $F_l\in\Sym^{2l}$ arbitrary.  Since $I$ has real coefficients, $I$ of the conjugated tuple is the conjugate of $I$; since $I$ is multihomogeneous, the scalars $(-1)^l$ come out as $(-1)^{\sum_lld_l}$; since $I$ is invariant, $w^{-1}$ can be dropped.  Hence $I\bigl((\sigma_lF_l)_l\bigr)=\overline{I\bigl(((-1)^l\,w^{-1}\cdot F_l)_l\bigr)}=(-1)^{\sum_ll\,d_l}\;\overline{I\bigl((F_l)_l\bigr)}$.  For real $f_l$, $F_l=F_{f_l}$ satisfies $\sigma_lF_l=F_l$ by (i), and the claim follows.

(iv) The map $-1$ sends a homogeneous polynomial $f$ of degree $l$ to $f(-x)=(-1)^lf(x)$.  Hence it multiplies $F_{f_l}$ by $(-1)^l$, and by multihomogeneity it multiplies $I$ by $\prod_l(-1)^{ld_l}=(-1)^{\sum_lld_l}$.  Every reflection is $-R$ with $R\in\SO(3)$, and $I$ is $\SO(3)$-invariant, so a reflection acts on $I$ as $-1$ does.  The even and odd parts of an invariant are the sums of its multihomogeneous components with $\sum_lld_l$ even and odd, respectively.
\end{proof}
\begin{defn}
\label{def:pseudo}
A multihomogeneous invariant $I$ with real coefficients is an \textbf{$\Or(3)$-invariant} if $\sum_lld_l$ is even, and an \textbf{$\Or(3)$-pseudo-invariant}, or simply a \textbf{pseudo-invariant}, if $\sum_lld_l$ is odd.
\end{defn}

By \cref{lem:reality}(iii) and (iv), an $\Or(3)$-invariant takes real values on real harmonics and is unchanged by reflections, while a pseudo-invariant takes purely imaginary values on real harmonics and changes sign under every reflection.  For a pseudo-invariant $I$, the function $-iI$ is real on real harmonics and changes sign under reflections.

Part (iii) is used repeatedly below: the degree-15 skew invariant $R$ of the sextic is real on sextics with real coefficients but purely imaginary on the real form arising from degree-three harmonics, and part (iv) identifies this with reflection parity.

\subsection{The harmonic--binary-form correspondence}
\label{sec:correspondence}

The \textbf{Bombieri--Weyl norm} of $F=\sum_{k=0}^{2l}\binom{2l}ka_kz_0^{2l-k}z_1^k$ is (see~\cite{BeauzamyBombieriEnfloMontgomery90})
\[
\norm F_{\rm BW}^2=\sum_{k=0}^{2l}\binom{2l}k\abs{a_k}^2 .
\]
Put $\mathcal Y_{l0}=Y_{l0}$ and, for $1\le p\le l$,
\[
\mathcal Y_{l,\pm p}=\frac{(\mp1)^p}{\sqrt2}\bigl(Y_{lp}\pm i\,Y_{l,-p}\bigr).
\]
The $\mathcal Y_{lm}$, $\abs m\le l$, form an orthonormal basis of $\Harm_l^\C$ for the Hermitian inner product of $L^2(\Sph,d\omega)$, and $\mathcal Y_{lm}$ is a multiple of $P_l^{\abs m}(\cos\theta)e^{im\phi}$.  Let $V_L=\bigoplus_{l\le L}\Harm_l$ and $\Psi=\bigoplus_{l\le L}\Psi_l$.

\begin{thm}
\label{thm:correspondence}
Let $l\ge0$.  The following are true.
\begin{enumerate}[label=(\roman*)]
\item 
The substitution $P\mapsto P\circ v$ induces an isomorphism of graded algebras $\C[x_1,x_2,x_3]/(q)\cong\bigoplus_{l\ge0}\Sym^{2l}$; its restriction $\Psi_l\colon\Harm_l^\C\to\Sym^{2l}$ is a linear isomorphism, and $\Psi_l^{-1}(F)=\pi(P)$ for every $P\in P_l(\C^3)$ with $P\circ v=F$, where
\[
\pi(P)=\sum_{j=0}^{\lfloor l/2\rfloor}\frac{(-1)^j\,q^j\,\Delta^jP}{2^j\,j!\,\prod_{i=1}^{j}(2l-2i+1)} .
\]

\item 
For $g\in\SL_2(\C)$ and $f\in\Harm_l^\C$ one has $\Psi_l(\rho(g)\cdot f)=g\cdot\Psi_l(f)$, and every $\SL_2(\C)$-equivariant linear map $\Harm_l^\C\to\Sym^{2l}$ is a scalar multiple of $\Psi_l$.

\item 
For $\abs m\le l$,
\[
\Psi_l(\mathcal Y_{lm})=K_{lm}\,z_0^{l-m}z_1^{l+m},\qquad K_{lm}=\frac{(2l)!}{l!}\sqrt{\frac{2l+1}{(l-m)!\,(l+m)!}} .
\]
For $f=\sum_{\abs m\le l}\eta_m\mathcal Y_{lm}\in\Harm_l^\C$ and $F_f=\sum_k\binom{2l}ka_kz_0^{2l-k}z_1^k$ one has $\eta_m=l!\,a_{l+m}/\sqrt{(2l+1)\,(l-m)!\,(l+m)!}$.  Let $f=\sum_{\abs m\le l}c_{lm}Y_{lm}\in\Harm_l$, let $F_f=\sum_k\binom{2l}ka_kz_0^{2l-k}z_1^k$, and put $\nu_{lp}=\sqrt{(2l+1)(l-p)!\,(l+p)!/2}\,/\,l!$ for $1\le p\le l$.  Then
\[
\begin{split}
a_l&=\sqrt{2l+1}\,c_{l0},\\
a_{l-p}&=\nu_{lp}\bigl(c_{lp}+i\,c_{l,-p}\bigr),\\
a_{l+p}&=(-1)^p\,\nu_{lp}\bigl(c_{lp}-i\,c_{l,-p}\bigr).
\end{split}
\]
Conversely, $c_{l0}=a_l/\sqrt{2l+1}$, $c_{lp}=\Re(a_{l-p})/\nu_{lp}$ and $c_{l,-p}=\Im(a_{l-p})/\nu_{lp}$.

\item 
For every $f\in\Harm_l^\C$, $\norm{\Psi_l(f)}_{\rm BW}^2=(2l+1)\binom{2l}l\,\norm f_{L^2(d\omega)}^2$.  Hence $\widehat\Psi_l=\bigl((2l+1)\binom{2l}l\bigr)^{-1/2}\,\Psi_l$ is unitary from $L^2(\Sph,d\omega)$ to the Bombieri--Weyl norm.  For real $f\in\Harm_l$ one has $\norm{F_f}_{\rm BW}^2=(-1)^l(F_f,F_f)_{2l}$, where $(F,F)_{2l}=\sum_k(-1)^k\binom{2l}ka_ka_{2l-k}$ is the transvectant~\cref{eq:transvectant}.

\item 
For $f\in\Harm_l^\C$ one has $\Psi_l(\bar f)=\sigma_l(\Psi_lf)$.  Hence $\Psi_l(\Harm_l)=\Fix(\sigma_l)$, the set of forms with $a_{2l-k}=(-1)^{l+k}\bar a_k$ for all $k$.
\end{enumerate}
\end{thm}

\begin{proof}
(i) and (ii) are classical, see~\cite{Cartan1938,Weyl46,Vilenkin68}: (i) is the Fischer decomposition together with the parameterization of $N$ by $v$, and the formula for $\pi(P)$ follows from $\Delta(q^jG)=2j(2j+2d+1)\,q^{j-1}G+q^j\Delta G$ for $G$ homogeneous of degree $d$; (ii) follows from $v(g^{-1}z)=\rho(g)^{-1}v(z)$ and Schur's lemma.

(iii) Since $\rho(g_\theta)=R_{-\theta}$ and $\mathcal Y_{lm}$ depends on $\phi$ through $e^{im\phi}$, $\rho(g_\theta)\cdot\mathcal Y_{lm}=e^{im\theta}\mathcal Y_{lm}$.  By (ii), $\Psi_l(\mathcal Y_{lm})$ is an eigenvector of $g_\theta$ with eigenvalue $e^{im\theta}$, and $g_\theta\cdot z_0^pz_1^{2l-p}=e^{i\theta(l-p)}z_0^pz_1^{2l-p}$; the characters are distinct, so $\Psi_l(\mathcal Y_{lm})$ is a multiple of $z_0^{l-m}z_1^{l+m}$.  Let $p=\abs m$.  The homogeneous extension of $P_l^p(\cos\theta)e^{\pm ip\phi}$ is $(x_1\pm ix_2)^pG_{lp}$, where $G_{lp}=\abs x^{l-p}P_l^{(p)}(x_3/\abs x)$ is a polynomial in $x_3$ and $q$.  Its part free of $q$ is $\lambda_{lp}x_3^{l-p}$, where $\lambda_{lp}=(2l)!/(2^l\,l!\,(l-p)!)$ is the leading coefficient of $P_l^{(p)}$.  By the substitutions following \cref{def:psi}, $x_1+ix_2\mapsto-2z_1^2$, $x_1-ix_2\mapsto2z_0^2$, $x_3\mapsto2z_0z_1$ and $q\mapsto0$.  Hence
\[
\Psi_l\bigl((x_1\pm ix_2)^pG_{lp}\bigr)=\frac{(\mp1)^p\,(2l)!}{l!\,(l-p)!}\,z_0^{l\mp p}z_1^{l\pm p}.
\]
By~\cref{eq:basis}, $Y_{lp}\pm iY_{l,-p}=N_{lp}(x_1\pm ix_2)^pG_{lp}$ for $p\ge1$ and $Y_{l0}=\sqrt{2l+1}\,G_{l0}$.  Multiplying by the factors in the definition of $\mathcal Y_{lm}$ gives $K_{lm}$.  Next, $f=\sum_m\eta_m\mathcal Y_{lm}$ with $\eta_0=c_{l0}$, $\eta_p=(-1)^p(c_{lp}-ic_{l,-p})/\sqrt2$ and $\eta_{-p}=(c_{lp}+ic_{l,-p})/\sqrt2$.  Comparing coefficients gives $a_{l+m}=K_{lm}\eta_m/\binom{2l}{l+m}$ for every $f=\sum_m\eta_m\mathcal Y_{lm}\in\Harm_l^\C$.  Since $K_{lm}/\binom{2l}{l+m}=\sqrt{(2l+1)(l-m)!\,(l+m)!}\,/\,l!$, this is the formula for $\eta_m$.  Moreover $K_{lp}/(\sqrt2\binom{2l}{l+p})=\nu_{lp}$.  This gives the formulas for $a_k$.  Since the $c_{lm}$ are real, the formulas can be inverted as stated.

(iv) Write $f=\sum_m\eta_m\mathcal Y_{lm}$.  By (iii), $a_{l+m}=K_{lm}\eta_m/\binom{2l}{l+m}$, and $K_{lm}^2/\binom{2l}{l+m}=(2l+1)\binom{2l}l$ for every $m$.  Summing over $m$ and using the orthonormality of the $\mathcal Y_{lm}$ gives the first identity.  By~\cref{eq:transvectant} with $F=G$ and $k=2l$, the $(2l)$-th transvectant is $(F,F)_{2l}=\sum_k(-1)^k\binom{2l}ka_ka_{2l-k}$.  For real $f$, (v) gives $a_{2l-k}=(-1)^{l+k}\bar a_k$, hence $(F_f,F_f)_{2l}=(-1)^l\norm{F_f}_{\rm BW}^2$.  Since $\Psi_l$ is a linear isomorphism, the norm identity and polarization show that $\widehat\Psi_l$ is unitary.

(v) This is \cref{lem:reality}(i) together with~\cref{eq:sigma-coeff}.
\end{proof}
\begin{rem}
\label{rem:power}
For $l=1,2,3$, \cref{thm:correspondence}(iv) expresses the band power through the classical invariants.  For $l=1$ one has $\disc(Q_a)=-2(Q_a,Q_a)_2$, so $\norm{f_1}^2=\disc(Q_a)/12$.  For $l=2$ one has $(F,F)_4=2I$ with $I$ as in~\cref{eq:IJ}, so $\norm{f_2}^2=I(F_{f_2})/15$.  For $l=3$ direct expansion gives $I_2=-120\,(F,F)_6$ with $I_2$ the Igusa--Clebsch invariant of \cref{sec:degree-three}, so $J_2(F_{f_3})=2100\,\norm{f_3}^2$.  In particular $J_2>0$ on every nonzero real band-three harmonic.  The squared band norms are therefore the phased quadratic invariants of the forms $F_{f_l}$, up to the constants of \cref{thm:correspondence}(iv).  We use $\Psi_l$, which keeps the classical constants of the invariants; the unitary map $\widehat\Psi_l$ is the normalization for comparing bands numerically.
\end{rem}

\begin{cor}
\label{cor:maxwell}
Let $l\ge0$ and let $f\in\Harm_l$ be real and nonzero.  Then $F_f=\lambda Q_{a_1}\cdots Q_{a_l}$ and $f=\pi(\lambda\ip{a_1}x\cdots\ip{a_l}x)$ with $\lambda\in\R^\times$ and $a_k\in\R^3\setminus\{0\}$, and the lines $\R a_k$, the \textbf{Maxwell axes} of $f$, are determined by $f$ up to order, with multiplicity.  The root divisor of $F_f$, transported to $\Sph$ by $\bar m$, consists of the $l$ antipodal pairs $\pm a_k/\abs{a_k}$.  Every effective antipodally invariant divisor of degree $2l$ on $\Sph$, counted with multiplicity, is the root divisor of $F_f$ for a nonzero real $f$, unique up to $\R^\times$.  For $R\in\SO(3)$ the roots of $F_{R\cdot f}$ are the images under $R$ of the roots of $F_f$.  For $-1\in\Or(3)$ one has $F_{(-1)\cdot f}=(-1)^lF_f$.
\end{cor}

\begin{proof}
This is the theorem of Maxwell and Sylvester~\cite{Maxwell1873,Sylvester1876}.  In the present normalization it is \cref{lem:reality}(ii), since $\Psi_l\circ\pi$ is evaluation at $v(z)$, together with~\cref{eq:hopf} and \cref{thm:correspondence}(v); the equivariance statements are \cref{thm:correspondence}(ii) and \cref{lem:reality}(iv).
\end{proof}

The rotation quotient of a tuple of harmonics $(f_0,\dots,f_L)$ is thus the joint invariant theory of the tuple $(F_{f_0},\dots,F_{f_L})$ of binary forms of degrees $0,2,\dots,2L$.  Invariants of the separate forms do not record the relative orientation of the bands; joint invariants are needed for that (\cref{sec:bands}).

\section{Invariants and Separation}
\label{sec:invariants}

Fix $L\ge1$ and let $V_L=\bigoplus_{l=0}^L\Harm_l$ with $\SO(3)$ acting diagonally; its elements $c=(f_0,\dots,f_L)$ are the truncated expansions of \cref{sec:surfaces}.  The components $f_1\in\Harm_1$ and $f_2\in\Harm_2$ are called the \textbf{dipole} and the \textbf{quadrupole}, and the tuple of band norms $(\norm{f_l})_{l\le L}$ is the \textbf{power spectrum}, see~\cite{KazhdanFunkhouserRusinkiewicz03}.  Through $\Psi=\bigoplus\Psi_l$ of \cref{def:psi}, $V_L\otimes\C$ is the space of tuples of binary forms of degrees $0,2,\dots,2L$ with $\SL_2(\C)$ acting diagonally, and by \cref{thm:correspondence} and the Zariski density of $\SU(2)$ in $\SL_2(\C)$,
\[
\R[V_L]^{\SO(3)}\otimes\C\;\cong\;\mathcal R_L=\C\bigl[\textstyle\bigoplus_{l\le L}\Sym^{2l}\bigr]^{\SL_2(\C)},
\]
the ring of joint invariants of even binary forms.  It is finitely generated (see~\cite{Hilbert1890}), with multihomogeneous generators having rational coefficients.  By \cref{lem:reality}(iii), multiplying a generator by $i^{-\sum ld_l}$ gives a real invariant on $V_L$; these \textbf{phased} generators generate $\R[V_L]^{\SO(3)}$.

\begin{thm}[Schwarz; Procesi--Schwarz]
\label{thm:sep-radial}
Let $I_1,\dots,I_N$ be the phased real forms of a multihomogeneous generating set of $\mathcal R_L$, of total degrees $w_1,\dots,w_N$, and $\mathcal I=(I_1,\dots,I_N)\colon V_L\to\R^N$ its \textbf{Hilbert map}.  Then $\mathcal I(c)=\mathcal I(c')$ iff $c'=R\cdot c$ for some $R\in\SO(3)$, $\mathcal I(c)=0$ iff $c=0$, and $\mathcal I(V_L)$ is a closed semialgebraic set homeomorphic to $V_L/\SO(3)$.  The separation statement holds for any separating subset.
\end{thm}

This is the theorem for a compact group acting linearly on a real vector space, see~\cite{Schwarz75,ProcesiSchwarz85}, applied through the isomorphism $\R[V_L]^{\SO(3)}\otimes\C\cong\mathcal R_L$; since the orbit $\{0\}$ is closed and the $I_k$ have positive degree, $\mathcal I^{-1}(0)=\{0\}$.

\subsection{Scale}
\label{sec:scale}
A dilation $c\mapsto\lambda c$, $\lambda>0$, multiplies an invariant of total degree $w_k$ by $\lambda^{w_k}$, so the class of $\mathcal I(c)$ in the positive weighted projectivization
\[
\PP^+_w(\R)=(\R^N\setminus0)/\R_{>0},\qquad\lambda\cdot x=(\lambda^{w_1}x_1,\dots,\lambda^{w_N}x_N),
\]
is invariant under rotation and positive scaling.  We also write $\PP_w(\R)=(\R^N\setminus0)/\R^\times$ for the quotient by all nonzero scalars, with the same action.  When $\mathcal I$ separates, it is complete on nonzero truncations: equal classes mean $I_k(c')=\lambda^{w_k}I_k(c)=I_k(\lambda c)$, so $c'\sim\lambda c$.  We quotient by $\R_{>0}$ rather than $\R^\times$ because odd weights would otherwise identify $c$ with $-c$, which on odd bands agrees with the mirror image.  Scale is split off, not discarded: the mean radius $f_0$, or a coefficient norm, is retained as a size variable.  A canonical representative is given by the weighted sup-gauge
\begin{equation}
\label{eq:gauge}
H_w(x)=\max_k\abs{x_k}^{1/w_k},\qquad\hat x_k=\frac{x_k}{H_w(x)^{w_k}},
\end{equation}
so that $\max_k\abs{\hat x_k}^{1/w_k}=1$; $H_w$ is the archimedean factor of the weighted height of~\cite{2019-1}.  It preserves all signs, including those of pseudo-invariants, and the bounded quantity $\abs{\hat I_k}^{1/w_k}\in[0,1]$ is a coordinatewise diagnostic of approach to the divisor $I_k=0$.

\subsection{Band-by-band invariants}
\label{sec:bands}
For $F=F_f$ the binomial coefficients $a_k$ are given in terms of the coefficients $c_{lm}$ of $f$ by \cref{thm:correspondence}(iii), and \cref{lem:reality}(iii) applies to everything built from transvectants~\cref{eq:transvectant}.

The invariants of band $l$ are the classical invariants of binary forms of degree $2l$, see~\cite{Olver99}; we use degree two here and degree three in \cref{sec:degree-three}.

\bmhead{Degree two}
A real harmonic is $f=x^{\top}Qx$ with $Q$ traceless symmetric, and its binary form $F_f$ (\cref{def:psi}) is a quartic, with binomial coefficients $a_0,\dots,a_4$, whose invariant ring is freely generated by
\begin{equation}
\label{eq:IJ}
I=a_0a_4-4a_1a_3+3a_2^2,\quad J=\det\begin{pmatrix}a_0&a_1&a_2\\a_1&a_2&a_3\\a_2&a_3&a_4\end{pmatrix},
\end{equation}
of degrees $2,3$, with $\disc F=I^3-27J^2$ the normalized discriminant.  The quartic has a repeated root if and only if $\disc F=0$.

\begin{exa}
\label{ex:quartic}
For $Q=\diag(\lambda_1,\lambda_2,\lambda_3)$, $\sum\lambda_k=0$, \cref{thm:correspondence}(iii) gives $a_0=a_4=\lambda_1-\lambda_2$, $a_1=a_3=0$, $a_2=\lambda_3$, hence $I(F_f)=2\tr Q^2$ and $J(F_f)=-4\det Q$.  As a point of $\PP^+_{(2,3)}(\R)$ the band-two shape is $[I:J]$; its image is the closed interval $I^3\ge27J^2$, whose endpoints are the axially symmetric quadrupoles.
\end{exa}

\bmhead{Degree two with the lower bands}
Band-by-band invariants do not record the relative orientation of the bands; the first case is $L=2$.

\begin{thm}
\label{thm:L2}
Write $c=(f_0,a,Q)\in V_2=\Harm_0\oplus\Harm_1\oplus\Harm_2$ with $f_0\in\R$, $a\in\R^3$, and $Q\in\Sym^2_0(\R^3)$, the traceless symmetric $3\times3$ real matrices.  The seven invariants
\[
f_0,\quad\abs a^2,\quad\tr Q^2,\quad\det Q,\quad a^{\top}Qa,\quad a^{\top}Q^2a,\quad\varepsilon:=\det(a,Qa,Q^2a),
\]
of degrees $1,2,2,3,3,4,6$, separate the $\SO(3)$-orbits in $V_2$.  The first six are $\Or(3)$-invariants and $\varepsilon$ is a pseudo-invariant.  They satisfy the relation
\[
\varepsilon^2=\det\Gram(a,Qa,Q^2a),
\]
whose right side is a polynomial in the other six by $Q^3=\tfrac12(\tr Q^2)Q+(\det Q)\,1$.  In binary-form language, with $Q_a$ as in~\cref{eq:Qx} and transvectants as in~\cref{eq:transvectant}, $\abs a^2=\tfrac14\disc Q_a$, $\tr Q^2$ and $\det Q$ are $I$ and $J$ of \cref{ex:quartic} up to constants, $a^{\top}Qa$ is proportional to $(Q_a^2,F)_4$, and $a^{\top}Q^2a$ and $\varepsilon$ are joint invariants of bidegrees $(2,2)$ and $(3,3)$ in $(Q_a,F)$.
\end{thm}

\begin{proof}
For traceless $Q$, $\tr Q^3=3\det Q$, so $\abs a^2$, $\tr Q^2$, $\det Q$, $a^{\top}Qa$, $a^{\top}Q^2a$ form the integrity basis of Spencer and Rivlin~\cite{SpencerRivlin62} for the $\Or(3)$-invariants of a vector and a symmetric tensor, and together with $f_0$ they generate $\R[V_2]^{\Or(3)}$; here $-1\in\Or(3)$ acts by $(f_0,a,Q)\mapsto(f_0,-a,Q)$, which changes the sign of $\varepsilon$.  By separation for compact groups~\cite{Schwarz75}, if the first six invariants agree on $c$ and $c'$, then $c'$ lies in the $\SO(3)$-orbit of $c$ or of $c^*=(f_0,-a,Q)$.  In the second case $\varepsilon(c')=\varepsilon(c^*)=-\varepsilon(c)$, so agreement of $\varepsilon$ forces $\varepsilon(c)=0$; then, in an eigenbasis of $Q$, either $Q$ has a repeated eigenvalue or some coordinate $a_i$ vanishes, and in both cases a rotation commuting with $Q$ sends $a$ to $-a$, so $c^*\in\SO(3)\cdot c$.  The relation is the Gram identity, and $Q^3$ reduces by Cayley--Hamilton for traceless $Q$.
\end{proof}

The seven invariants vanish simultaneously only at $c=0$ and are weighted homogeneous, so the map they define is proper.  Hence they realize $V_2/\SO(3)$, of dimension six, as a closed semialgebraic subset of $\R^7$, contained in the hypersurface $\varepsilon^2=\det\Gram(a,Qa,Q^2a)$; the sign of $\varepsilon$ is the chirality of the dipole--quadrupole pair (\cref{prop:chirality}(ii)).  After scale,
\[
[f_0:\abs a^2:\tr Q^2:\det Q:a^{\top}Qa:a^{\top}Q^2a:\varepsilon]\in\PP^+_{(1,2,2,3,3,4,6)}(\R)
\]
is a complete scale-free invariant of the truncation.

\begin{rem}
\label{rem:frames}
Let $K=\{\diag(\epsilon_1,\epsilon_2,\epsilon_3):\epsilon_i=\pm1,\ \epsilon_1\epsilon_2\epsilon_3=1\}\cong(\Z/2)^2$, the half-turns about the coordinate axes together with the identity; in the basis $\{Y_{lm}\}$ every element of $K$ acts on each coordinate of $V_L$ by a sign.  On the open dense set $V_L^\circ$ where $Q$ has eigenvalues $\lambda_1>\lambda_2>\lambda_3$, an ordered orthonormal eigenframe $E\in\SO(3)$ of $Q$ is unique up to $E\mapsto E\kappa$, $\kappa\in K$, so the aligned tuple $\tilde c=E^{-1}\cdot c$ determines the orbit of $c$ up to the action of $K$, and the $K$-invariant monomials of degree at most three in its $\dim V_L-3$ coordinates separate the orbits on $V_L^\circ$, since they generate the invariants of $K$ (the Davenport constant of $(\Z/2)^2$ is three, see~\cite{DerksenKemper15}).  These are the \textbf{frame coordinates}.  The complement of $V_L^\circ$ consists of the axial strata $\lambda_1=\lambda_2$ or $\lambda_2=\lambda_3$, detected by $\disc F_{f_2}=I^3-27J^2=0$, and of $Q=0$, detected by $I=0$.  The frame has no continuous extension to the complement, since the stabilizer of an axial $Q$ is infinite while $K$ is finite, and near it the frame coordinates are ill-conditioned, with condition number governed by the reciprocal of the eigenvalue gap; the polynomial invariants have no such defect, by \cref{thm:stability}.
\end{rem}

\bmhead{Joint invariants}
For bands $(l_1,l_2,l_3)$ satisfying the triangle inequalities, put $k=l_1+l_2-l_3$, $F_l=F_{f_l}$, and
\begin{equation}
\label{eq:bispectral}
B_{l_1l_2l_3}=\bigl((F_{l_1},F_{l_2})_k,F_{l_3}\bigr)_{2l_3}.
\end{equation}
These are the bispectral invariants of~\cite{Kakarala12} as transvectants, real when $l_1+l_2+l_3$ is even and pseudo-invariants when it is odd.  When two bands coincide, $B_{l_1ll}$ is proportional to the unique invariant pairing of $\Harm_{l_1}\otimes\Harm_l\otimes\Harm_l$, whose symmetry under exchange of the two $\Harm_l$ factors is $(-1)^{l_1+2l}$; hence $B_{l_1ll}=0$ for odd $l_1$.  A cubic pseudo-invariant therefore needs three distinct bands with odd sum, the smallest being $B_{234}$, so none occurs for $L\le3$; $P_{123}=((F_1,F_2)_1,(F_3,F_3)_4)_4$ is a pseudo-invariant of degree four on $V_3$.
\section{Chirality}
\label{sec:chirality}

A tuple $c=(f_l)_l\in V_L$ is \textbf{achiral} if its \textbf{mirror image} $c^*=((-1)^lf_l)_l$, the image of $c$ under $-1\in\Or(3)$, lies in $\SO(3)\cdot c$; equivalently, if $S\cdot c=c$ for some $S\in\Or(3)\setminus\SO(3)$.  The pseudo-invariants detect this.

\begin{prop}[Chirality criterion]
\label{prop:chirality}
Let $c\in V_L$.  Then $c$ is achiral iff every pseudo-invariant in $\mathcal R_L$ vanishes at $\Psi c$.  In particular:
\begin{enumerate}[label=(\roman*)]
\item every $c$ supported on even bands is achiral, and so is every $c$ supported on band one;
\item a tuple supported on bands one and two, $f_1=\ip a\cdot$, $f_2=x^{\top}Qx$, is achiral iff $\det(a,Qa,Q^2a)=0$, i.e.\ iff $a$ lies in a plane spanned by two eigenvectors of $Q$.
\end{enumerate}
\end{prop}

\begin{proof}
Decompose a generating set into parity components; by \cref{thm:sep-radial}, $c$ and $c^*$ are in the same orbit iff all components agree on both, and even components agree automatically while odd ones change sign.  (i) On even bands all $ld_l$ are even; on band one the ring is generated by the discriminant.  (ii) In an eigenbasis of $Q$ with distinct eigenvalues, $\det(a,Qa,Q^2a)=a_1a_2a_3\prod_{i<j}(\lambda_j-\lambda_i)$, which vanishes iff some $a_i=0$; then the half-turn about $e_i$ commutes with $Q$ and sends $a\mapsto-a$, so $c$ is achiral.  Conversely if all $a_i\ne0$, every rotation commuting with $Q$ is in $K$ and none sends $a$ to $-a$.  If $Q$ has a repeated eigenvalue the determinant vanishes and $\Stab(Q)$ contains a rotation sending $a$ to $-a$.
\end{proof}

\subsection{Odd degree}
\label{sec:odd}

Let $n$ be odd and $0\ne f\in\Harm_n$.  By \cref{cor:maxwell}, $f=\pi(\lambda\ip{a_1}x\cdots\ip{a_n}x)$ with Maxwell axes $\R a_1,\dots,\R a_n$, counted with multiplicity.  A rotation $\rho$ with $\rho\cdot f=-f$ is called \textbf{reversing}.  Since $-1\in\Or(3)$ acts on $\Harm_n$ by $-1$, $f$ is achiral iff it admits a reversing rotation.  If $\rho$ permutes the axes, $\rho a_k=t_ka_{\pi(k)}$ with $t_k=\pm1$, then $\rho\cdot f\equiv\lambda\prod_k\ip{\rho a_k}x\pmod q$, so
\begin{equation}
\label{eq:signs}
\rho\cdot f=\Bigl(\prod_kt_k\Bigr)f,
\end{equation}
and the product does not depend on the choice of representatives, since rescaling $a_k\mapsto s_ka_k$ multiplies it by $\prod_ks_k/s_{\pi(k)}=1$.

\begin{thm}
\label{thm:chiral-odd}
Let $n$ be odd, $0\ne f\in\Harm_n$, and $F=F_f$.  The following are equivalent.
\begin{enumerate}[label=(\roman*)]
\item $f$ is achiral.
\item Every invariant of odd degree of binary forms of degree $2n$ vanishes at $F$.
\item $f$ admits a reversing rotation.
\item There is a rotation of even order $2k$ permuting the Maxwell axes of $f$ such that the axes in the plane orthogonal to its axis form, counted with multiplicity, an odd number of orbits.
\item $g\cdot F=-F$ for some $g\in\SL_2(\C)$.
\item For some $k\ge1$ and some $h\in\SU(2)$,
\[
h\cdot F=\sum_{j\equiv n+k\ (\mathrm{mod}\ 2k)}\binom{2n}ja_j\,z_0^{2n-j}z_1^j.
\]
\end{enumerate}
\end{thm}

\begin{proof}
(i)$\Leftrightarrow$(iii): every element of $\Or(3)\setminus\SO(3)$ is $-\rho$ with $\rho\in\SO(3)$, and $(-\rho)\cdot f=-\rho\cdot f$.  (i)$\Leftrightarrow$(ii): for a single band $\sum_lld_l=nd$, so the pseudo-invariants of band $n$ are the invariants of odd degree, and \cref{prop:chirality} applies.  (iii)$\Rightarrow$(v): if $\rho=\rho(g)$ with $g\in\SU(2)$, then $g\cdot F=\Psi_n(\rho\cdot f)=-F$ by \cref{thm:correspondence}(ii).  (v)$\Rightarrow$(ii): $I(F)=I(g\cdot F)=I(-F)=-I(F)$ for $I$ of odd degree.

(iii)$\Rightarrow$(iv).  Let $\rho\cdot f=-f$.  Since $\rho$ preserves $q$, it permutes the axes, and $\prod_kt_k=-1$ by~\cref{eq:signs}.  If all axes coincide, $f=\pi(\ip ax^n)$ up to a scalar, and the half-turn about any line $\ell\perp a$ reverses $a$; it satisfies (iv) with $k=1$, the single axis, of odd multiplicity $n$, lying in $\ell^\perp$.  Otherwise the axes span at least two directions.  Since $\rho$ permutes the axes, $\rho^{n!}$ fixes each of them, and a rotation fixing two non-parallel lines is the identity or a half-turn; so $\rho^{2\cdot n!}=1$ and $\rho$ has finite order $m$.  Let $\ell$ be the axis of $\rho$.  For a line $a\ne\R\ell$, $\rho^ja=a$ only if $\rho^j=1$, since a nontrivial rotation fixes only its axis; and $\rho^ja=-a$ only if $a\perp\ell$ and $\rho^j$ is the half-turn about $\ell$, since a rotation with eigenvalue $-1$ is a half-turn and its $(-1)$-eigenspace is orthogonal to its axis.  Hence the orbit of a line not in $\ell^\perp\cup\R\ell$ has size $m$, with product of signs $+1$ along it, and $\R\ell$ has sign $+1$.  If $m$ is odd, no power of $\rho$ is a half-turn, so the lines in $\ell^\perp$ also have orbits of size $m$ with product $+1$, and $\prod_kt_k=+1$, a contradiction.  So $m=2k$, a line in $\ell^\perp$ has an orbit of size $k$ with $\rho^ka=-a$, hence product $-1$, and $\prod_kt_k=(-1)^e$ with $e$ the number of orbits in $\ell^\perp$ counted with multiplicity.  Thus $e$ is odd.

(iv)$\Rightarrow$(iii).  For a rotation as in (iv) the same count gives $\prod_kt_k=(-1)^e=-1$, so $\rho\cdot f=-f$ by~\cref{eq:signs}.

(iii)$\Rightarrow$(vi).  Let $\rho$ be reversing of order $2k$, as (iii)$\Rightarrow$(iv) provides.  Choose $h\in\SU(2)$ with $\rho(h)^{-1}\rho\,\rho(h)$ a rotation about the $x_3$-axis, by an angle $\pi s/k$ with $s$ odd and prime to $2k$; by the formula $\rho(g_\theta)=R_{-\theta}$ of \cref{sec:binary} it is $\rho(g)$ with $g=\diag(\zeta,\zeta^{-1})$, $\zeta=e^{-i\pi s/2k}$, and $g\cdot(h\cdot F)=-h\cdot F$.  Writing $F'=h\cdot F=\sum_j\binom{2n}ja_jz_0^{2n-j}z_1^j$, one has
\[
(g\cdot F')(z)=F'(g^{-1}z)=\sum_j\binom{2n}ja_j\,\zeta^{2j-2n}\,z_0^{2n-j}z_1^j,
\]
and $\zeta^{2j-2n}=e^{-i\pi s(j-n)/k}=-1$ iff $s(j-n)\equiv k\pmod{2k}$ iff $j\equiv n+k\pmod{2k}$, since $s$ is odd and invertible modulo $2k$.  Hence $a_j=0$ unless $j\equiv n+k$.

(vi)$\Rightarrow$(v).  With $g=\diag(e^{-i\pi/2k},e^{i\pi/2k})$ the same computation gives $g\cdot(h\cdot F)=-h\cdot F$, so $(h^{-1}gh)\cdot F=-F$.
\end{proof}

In (vi) the affine form $F'(x,1)=\sum_j\binom{2n}ja_jx^{2n-j}$ has only exponents congruent to $n-k$ modulo $2k$, that is, $F'(x,1)=x^rK(x^{2k})$ with $r\equiv n-k\pmod{2k}$.

\begin{thm}
\label{thm:hyperelliptic}
Let $n\ge3$ be odd, let $f\in\Harm_n$ have distinct Maxwell axes, $F=F_f$, and let $C$ be the hyperelliptic curve $y^2=F(x,1)$ of genus $n-1$.  Then $f$ is achiral iff for some $k\ge1$
\[
C\cong y^2=x^rK(x^{2k}),\qquad r\in\{0,1\},\quad r\equiv n-k\pmod{2k},
\]
with $K$ squarefree of degree $(n-r)/k$ and $K(0)\ne0$; equivalently, iff $C$ has a reduced automorphism of order $2k$ whose lift to $\SL_2(\C)$ acts on $F$ by $-1$.
\end{thm}

\begin{proof}
If $f$ is achiral, \cref{thm:chiral-odd}(vi) gives coordinates in which $F(x,1)=x^rK(x^{2k})$ with $r\equiv n-k\pmod{2k}$; since the roots of $F$ are simple and $F$ has degree $2n$, the lowest exponent $r$ is $0$ or $1$, the highest is $2n-r$, $K(0)\ne0$, and $K$ is squarefree of degree $(n-r)/k$.  Conversely, if $C\cong C'$ with $C'$ of that form, then by~\cite{2004-3} the forms of degree $2n$ defining $C$ and $C'$ are $\GL_2(\C)$-equivalent up to a scalar, $F=c\,h\cdot F'$; the computation in the proof of \cref{thm:chiral-odd} gives $g\in\SL_2(\C)$ with $g\cdot F'=-F'$, and then $(hgh^{-1})\cdot F=-F$, so $f$ is achiral by \cref{thm:chiral-odd}(v).  The element $g$ induces a reduced automorphism of $C'$ of order $2k$, and conversely a reduced automorphism whose lift acts by $-1$ is condition (v).
\end{proof}

\begin{cor}
\label{cor:Ln}
In the setting of \cref{thm:hyperelliptic}, $f$ admits a reversing half-turn iff $C$ has an extra involution, that is, $C\in\mathcal L_{n-1}$ in the sense of~\cite{2003-3}.  If $f$ admits a reversing rotation of order $2k$ with $k$ odd, its $k$-th power is a reversing half-turn.
\end{cor}

\begin{proof}
We first show that every reduced automorphism of $C$ is induced by a rotation.  Let $B\subset\PP^1$ be the branch locus, $\bar G\subset\mathrm{PGL}_2(\C)$ the reduced automorphism group, that is, the stabilizer of $B$, and $\tau$ the antipodal map.  Since $\tau(B)=B$ by \cref{cor:maxwell}, $\tau$ normalizes $\bar G$, and $\Gamma=\langle\bar G,\tau\rangle$ is a finite group of conformal and anticonformal transformations of $\PP^1=\Sph$.  These extend to isometries of the open unit ball of $\R^3$ with its hyperbolic metric (see~\cite{Ahlfors81}); the center of the smallest closed ball containing a $\Gamma$-orbit is fixed by $\Gamma$, so there is a M\"obius transformation $h$ such that $h\Gamma h^{-1}$ fixes the origin, that is, $h\Gamma h^{-1}\subset\Or(3)$.  Then $h\tau h^{-1}$ is an orientation-reversing isometric involution of $\Sph$ without fixed points, hence the antipodal map, so $h$ commutes with $\tau$.  A M\"obius transformation commuting with the antipodal map is a rotation, so $h\in\SO(3)$ and $\bar G\subset h^{-1}\SO(3)h=\SO(3)$.  A rotation preserving $B$ permutes its antipodal pairs, that is, the Maxwell axes.

An involution of $\PP^1$ preserving the branch locus is, in a suitable coordinate, $x\mapsto-x$, with lift $g=\diag(i,-i)$, and $g\cdot z_0^{2n-j}z_1^j=(-1)^{n+j}z_0^{2n-j}z_1^j$.  If it fixes no branch point, then $F(-x,1)=\pm F(x,1)$ with $F(0,1)\ne0$, so $F(x,1)=K(x^2)$, the normal form of~\cite{2003-3} of a curve with an extra involution, and $g\cdot F=(-1)^nF=-F$.  Conversely, if $g\cdot F=-F$, only even powers of $x$ occur in $F(x,1)$; since the roots of $F$ are simple, $0$ and $\infty$ are not roots, and the involution fixes no branch point.  So an involution in the reduced automorphism group is extra iff its lift acts on $F$ by $-1$.  A reversing half-turn is $\rho(g)$ with $g\in\SU(2)$ and $g\cdot F=-F$ by \cref{thm:correspondence}(ii), so it induces an extra involution.  Conversely, an extra involution is induced by a half-turn $\rho(g)$, $g\in\SU(2)$, by the first paragraph, and $g\cdot F=-F$, so $\rho(g)\cdot f=-f$.  If $\rho$ is reversing of order $2k$ with $k$ odd, then $\rho^k$ is a half-turn and $\rho^k\cdot f=(-1)^kf=-f$.
\end{proof}

Let $\mathcal M^{\rm hyp}_g$ be the moduli space of hyperelliptic curves of genus $g$, and for odd $n\ge3$ let $\mathcal A_n\subset\mathcal M^{\rm hyp}_{n-1}$ be the locus of curves with a reduced automorphism of even order whose lift to $\SL_2(\C)$ acts on the defining form by $-1$; by \cref{thm:hyperelliptic} the curve of a real harmonic with distinct axes lies in $\mathcal A_n$ iff the harmonic is achiral, and the curves of $\mathcal A_n$ that arise from real harmonics are described in \cref{thm:real}.  For an even divisor $k$ of $n-1$ with $(n-1)/k$ odd, let $\mathcal S_{n,k}\subset\mathcal M^{\rm hyp}_{n-1}$ be the locus of the curves $y^2=xK(x^{2k})$ with $K$ squarefree of degree $(n-1)/k$ and $K(0)\ne0$.

\begin{thm}
\label{thm:strata}
Let $n\ge3$ be odd and $g=n-1$.  Then
\[
\mathcal A_n=\mathcal L_g\;\cup\bigcup_{k}\mathcal S_{n,k},
\]
the union over the even divisors $k$ of $n-1$ with $(n-1)/k$ odd.  Each $\mathcal S_{n,k}$ is irreducible of dimension $(n-1)/k-1$.  If $(n-1)/k=1$ then $\mathcal S_{n,k}\subset\mathcal L_g$; if $(n-1)/k\ge3$ the generic curve of $\mathcal S_{n,k}$ has reduced automorphism group $\Z/2k$, and $\mathcal S_{n,k}\not\subset\mathcal L_g$.  In particular $\mathcal A_n=\mathcal L_g$ iff $n-1$ is a power of two.
\end{thm}

\begin{proof}
By \cref{thm:hyperelliptic}, $\mathcal A_n$ is the union over $k\ge1$ of the loci of curves $y^2=x^rK(x^{2k})$ with $r\in\{0,1\}$, $r\equiv n-k\pmod{2k}$, $K$ squarefree with $K(0)\ne0$ and $\deg K=(n-r)/k$.  Since $n$ is odd, an odd $k$ forces $r=0$, and the curve $y^2=K(x^{2k})$ is of the form $y^2=\tilde K(x^2)$, hence lies in $\mathcal L_g$ by \cref{cor:Ln}.  An even $k$ forces $r=1$ and $n-1\equiv k\pmod{2k}$, that is, $k\mid n-1$ with $(n-1)/k$ odd, and the locus is $\mathcal S_{n,k}$.  This proves the decomposition.

Put $d=(n-1)/k$.  If $d=1$, then $C\colon y^2=x(x^{2k}+c)$, and $x\mapsto s/x$ with $s^k=c$ is an involution of $\PP^1$ preserving the branch locus: it exchanges $0$ and $\infty$ and sends a root of $x^{2k}=-c$ to a root of $x^{2k}=s^{2k}/(-c)=-c$.  Its fixed points satisfy $x^2=s$, hence $x^{2k}=c$, so they are not branch points, and $C\in\mathcal L_g$.

Let $d\ge3$, let $C\in\mathcal S_{n,k}$, and let $\bar G\subset\mathrm{PGL}_2(\C)$ be its reduced automorphism group, which contains $\bar g\colon x\mapsto\zeta_{2k}x$.  By the classification of the finite subgroups of $\mathrm{PGL}_2(\C)$ (see~\cite{Klein1884,2003-3}), a finite subgroup containing an element of order $2k\ge4$ is cyclic of order $2km$, dihedral of order $4km$, or, when $k=2$, isomorphic to $S_4$; the groups $A_4$ and $A_5$ contain no element of order four.  Each possibility with $\bar G\ne\langle\bar g\rangle$ confines $K$ to a proper closed subset of the squarefree polynomials of degree $d$.  If $\bar G$ is cyclic of order $2km$ with $m\ge2$, its elements fix the same two points $0,\infty$ as $\bar g$, so $F=xK''(x^{2km})$ in the same coordinates, that is, $K(u)=K''(u^m)$, which is impossible unless $m\mid d$ and is then a proper closed condition.  If $\bar G$ is dihedral, then $\langle\bar g\rangle$ lies in its rotation subgroup, because the elements outside it are involutions, and $\bar G$ contains an involution $x\mapsto s/x$; the form $F=x\prod_i(x^{2k}-u_i)$ is invariant up to a scalar under $x\mapsto s/x$ iff the set $\{u_i\}$ of roots of $K$ is invariant under $u\mapsto s^{2k}/u$.  Since $d$ is odd, exactly one root $u_0$ is fixed, $u_0^2=s^{2k}$, and the remaining $d-1$ roots pair off with $u_iu_{i'}=u_0^2$; these are $(d-1)/2\ge1$ independent conditions.  If $k=2$ and $\bar G\supset S_4$, the branch locus of $C$ is a union of orbits of $S_4$ on $\PP^1$, which have sizes $6$, $8$, $12$ and $24$, so these curves form a family of dimension at most $2n/24$, which is less than $d-1=(n-3)/2$ for $n\ge7$.  Hence the generic curve of $\mathcal S_{n,k}$ has reduced automorphism group $\langle\bar g\rangle\cong\Z/2k$.  Its only involution $\bar g^k\colon x\mapsto-x$ fixes the branch points $0$ and $\infty$, so $C$ has no extra involution, and $\mathcal S_{n,k}\not\subset\mathcal L_g$.

The curves of $\mathcal S_{n,k}$ are parametrized by the monic squarefree $K$ of degree $d$ with $K(0)\ne0$, an irreducible open subset of $\C^d$.  Two such curves are isomorphic iff their forms are $\GL_2(\C)$-equivalent up to a scalar (see~\cite{2004-3}), and an isomorphism between two generic curves conjugates $\langle\bar g\rangle$ to itself, hence lies in its normalizer, the group generated by $x\mapsto\lambda x$ and $x\mapsto1/x$.  The generic fiber of the parametrization is therefore one-dimensional, and $\mathcal S_{n,k}$ is irreducible of dimension $d-1$; for $d=1$ it is a point.

Finally, write $n-1=2^am$ with $m$ odd.  The admissible $k$ are $2^am'$ with $m'\mid m$, and $(n-1)/k=m/m'\ge3$ for some $m'$ iff $m>1$.  So $\mathcal A_n\ne\mathcal L_g$ iff $m>1$.
\end{proof}

\begin{cor}
\label{prop:n35}
Let $n\ge3$ be odd and let $f\in\Harm_n$ have distinct Maxwell axes.  If $n-1$ is a power of two, in particular for $n=3,5,9$, then $f$ is achiral iff its curve has an extra involution.  For $n=7$ the curves $y^2=xK(x^4)$, $K$ a cubic, form the two-dimensional locus $\mathcal S_{7,2}\subset\mathcal A_7$, whose generic member has no extra involution.
\end{cor}

\Cref{thm:strata} describes complex curves, while the harmonics are real.  In the coordinate $x=z_0/z_1$ of \cref{sec:harmonics} the antipodal map $[z]\mapsto[jz]$ is $x\mapsto-1/\bar x$, an antiholomorphic involution of $\PP^1$ without fixed points, and the root divisor of the form of a real harmonic is invariant under it (\cref{cor:maxwell}).

\begin{thm}
\label{thm:real}
Let $n\ge3$ be odd.
\begin{enumerate}[label=(\roman*)]
\item A hyperelliptic curve of genus $n-1$ is the curve of a real harmonic of degree $n$ with distinct Maxwell axes iff its branch locus is invariant under a fixed-point-free antiholomorphic involution of $\PP^1$.
\item Let $k$ be an even divisor of $n-1$ with $d=(n-1)/k$ odd, and let $C\colon y^2=xK(x^{2k})$ be a curve of $\mathcal S_{n,k}$.  If the root set of $K$ is invariant under $u\mapsto c/\bar u$ for some $c>0$, then $C$ is the curve of a real harmonic; the converse holds when the reduced automorphism group of $C$ is $\Z/2k$.
\item Let $d\ge3$.  The curves of real harmonics in $\mathcal S_{n,k}$ with reduced automorphism group $\Z/2k$ form a Zariski dense semialgebraic subset of $\mathcal S_{n,k}$ of real dimension $d-1$, with exactly $(d+1)/2$ connected components, distinguished by the number $1,3,\dots,d$ of roots $u$ of $K$ with $\abs u^2=c$.
\end{enumerate}
\end{thm}

\begin{proof}
(i) If $C$ is the curve of a real harmonic $f$, its branch locus is the root divisor of $F_f$, which is invariant under the antipodal map.  Conversely, let $\tau$ be a fixed-point-free antiholomorphic involution of $\PP^1$ preserving the branch locus $B$.  Then $\tau$ is a real structure on $\PP^1$ without real points, and any two such are conjugate under $\mathrm{PGL}_2(\C)$, since the real forms of $\PP^1$ are the real line and the conic without real points (see~\cite{GrossHarris81}); so there is $h\in\mathrm{PGL}_2(\C)$ with $h\tau h^{-1}$ the antipodal map, and $h(B)$ is antipodally invariant.  By \cref{lem:reality}(ii) the form with root divisor $h(B)$ is a complex multiple of a $\sigma_n$-fixed form, which is $F_f$ for a real harmonic $f$ with distinct axes by \cref{thm:correspondence}(v), and $C\cong y^2=F_f(x,1)$ since the two forms are $\GL_2(\C)$-equivalent up to a scalar (see~\cite{2004-3}).

(ii) Let $G=\langle x\mapsto\zeta_{2k}x\rangle$.  If the roots of $K$ are invariant under $u\mapsto c/\bar u$, then after the substitution $x\mapsto c^{1/4k}x$, which preserves the normal form and replaces $u$ by $c^{1/2}u$, they are invariant under $u\mapsto1/\bar u$; since $2k$ is even, $(-1/\bar x)^{2k}=1/\overline{x^{2k}}$, so the branch locus $\{0,\infty\}\cup\{x:K(x^{2k})=0\}$ is antipodally invariant, and (i) applies.  Conversely, let $C$ be the curve of a real harmonic with reduced automorphism group $G$, and let $\tau$ be the antiholomorphic involution of (i).  Then $\tau G\tau^{-1}$ is a group of holomorphic automorphisms of $\PP^1$ preserving the branch locus, hence equal to $G$, so $\tau$ permutes the fixed points $0,\infty$ of $G$; having no fixed points, it exchanges them, $\tau(x)=c/\bar x$ with $c\in\C^\times$, and $\tau^2=1$ forces $c\in\R$ while the absence of fixed points forces $c<0$.  After $x\mapsto\abs c^{1/2}x$, $\tau$ is the antipodal map, and the invariance of the branch locus says that the roots of $K$ are invariant under $u\mapsto1/\bar u$.

(iii) Let $\mathcal K$ be the set of monic squarefree $K$ of degree $d$ with $K(0)\ne0$, and $\mathcal K_1\subset\mathcal K$ the subset of those whose root set is invariant under $u\mapsto1/\bar u$; by (ii) and the substitution $x\mapsto c^{1/4k}x$, the curves in question are the images of $\mathcal K_1$ under the parametrization $\pi\colon\mathcal K\to\mathcal S_{n,k}$ of the proof of \cref{thm:strata}, restricted to the open set where the reduced automorphism group is $\Z/2k$.  The set $\mathcal K_1$ is the fixed locus of the antiholomorphic involution $K\mapsto u^d\,\overline{K}(1/u)/\overline{K(0)}$ of $\mathcal K$, hence a totally real submanifold of real dimension $d=\dim_\C\mathcal K$; a holomorphic function vanishing on it vanishes identically, so $\mathcal K_1$ is Zariski dense in $\mathcal K$, and $\pi(\mathcal K_1)$ is Zariski dense in $\mathcal S_{n,k}$ because $\pi$ is dominant.  It is semialgebraic by the Tarski--Seidenberg theorem.

The involution $u\mapsto1/\bar u$ fixes the roots of $K$ on the unit circle and pairs the others; since $d$ is odd, the number $m$ of roots on the circle is odd.  Thus $\mathcal K_1=\bigsqcup_{m}\mathcal K_1^{(m)}$, $m=1,3,\dots,d$, and $\mathcal K_1^{(m)}$ is the product of the configuration spaces of $m$ unordered distinct points on the circle and of $(d-m)/2$ unordered distinct points in the punctured open unit disc, both connected; so $\mathcal K_1^{(m)}$ is connected, of real dimension $m+(d-m)=d$.  Two curves of $\pi(\mathcal K_1)$ with reduced automorphism group $\Z/2k$ are isomorphic iff their forms are related by an element of the normalizer of $G$, $x\mapsto\lambda x$ or $x\mapsto\lambda/x$, and such an element preserves the normalization $c=1$ iff $\abs\lambda=1$; it then acts on $u$ by $u\mapsto e^{i\theta}u$ or $u\mapsto e^{i\theta}/u$ and preserves the unit circle, so $m$ is an invariant of the curve, the fibers of $\pi$ on $\mathcal K_1^{(m)}$ are one-dimensional, and the images $\pi(\mathcal K_1^{(m)})$ are pairwise disjoint, connected, and of real dimension $d-1$.  Finally $m$ is locally constant on $\pi(\mathcal K_1)$: along a continuous family the roots of $K$ move continuously, and a root can leave or reach the unit circle only by colliding with its partner, which is excluded by squarefreeness.

It remains to see that removing the curves with a larger reduced automorphism group leaves each $\pi(\mathcal K_1^{(m)})$ connected.  By the proof of \cref{thm:strata}, the removed set is contained in the union of the dihedral locus, where the root set of $K$ is invariant under some $u\mapsto\lambda/u$, of complex dimension $(d+1)/2$, and the cyclic loci $K(u)=K''(u^p)$ with $p\mid d$, $p\ge3$, of complex dimension $d/p$.  A curve whose reduced automorphism group contains $S_4$ lies in the dihedral locus: the normalizer of $\langle\bar g\rangle$ in $S_4$ is dihedral of order $8$, and its involutions outside $\langle\bar g\rangle$ exchange the fixed points $0,\infty$ of $\bar g$, so they have the form $x\mapsto s/x$ and act on the roots of $K$ by $u\mapsto s^4/u$.  The intersection of a complex subvariety of dimension $e$ of $\mathcal K$ with the totally real submanifold $\mathcal K_1$ has real dimension at most $e$, so in $\mathcal K_1$ these sets have real codimension at least $(d-1)/2$ and $d-d/p\ge2$.  For $d\ge5$ both are at least $2$, and the complement of a closed semialgebraic set of codimension at least $2$ in the connected manifold $\mathcal K_1^{(m)}$ is connected.  Let $d=3$, and normalize by the rotations $u\mapsto e^{i\theta}u$; the cyclic locus is then contained in the dihedral locus.  If $m=3$, the dihedral condition says that some root is fixed by an inversion exchanging the other two, that is, two of the three arcs between consecutive roots are equal.  The triples of pairwise distinct arcs form six open chambers, permuted transitively by the cyclic relabellings and by $u\mapsto1/u$, which reverses the cyclic order, so their image is connected.  If $m=1$, normalize the root on the circle to $1$ and write $w$, $1/\bar w$ for the others, $0<\abs w<1$.  An inversion fixing $w$ or $1/\bar w$ would force $\abs w=1$, so the dihedral condition reads $w/\bar w=1$, that is, $w\in\R$.  The two components of $\{0<\abs w<1\}\setminus\R$ are exchanged by $u\mapsto1/u$, which replaces $w$ by $\bar w$, so the image is again connected.  Hence the $\pi(\mathcal K_1^{(m)})$, with these curves removed, are the connected components.
\end{proof}

\begin{exa}
\label{exa:n7}
Let $n=7$ and let the Maxwell axes be $e_3$ and the three pairs of orthogonal lines in the $(e_1,e_2)$-plane spanned by $(1,0,0),(0,1,0)$, $(3,1,0),(-1,3,0)$ and $(2,5,0),(-5,2,0)$.  The rotation by $\pi/2$ about $e_3$ permutes the axes with three orbits in the plane $e_3^\perp$, so $f$ is achiral by \cref{thm:chiral-odd}(iv), and $F(x,1)=xK(x^4)$ with $K$ a cubic.  The rotations permuting the axis set form the group $\Z/4$ generated by this rotation, whose half-turn preserves $f$, so $f$ has no reversing half-turn; and the reduced automorphism group of $C\colon y^2=xK(x^4)$ is $\Z/4$, its only involution fixing the branch points $0$ and $\infty$, as \cref{thm:strata} predicts for a generic member of $\mathcal S_{7,2}$.  The three roots of $K$ are $1$, $\bigl((1-2i)/(2-i)\bigr)^2$ and $\bigl((5+2i)/(5-2i)\bigr)^2$, all of modulus one, so $C$ lies in the component $m=3$ of \cref{thm:real}(iii).  The first group was computed in exact arithmetic by enumerating the rotations determined by the images of two axes, the second numerically by enumerating the M\"obius transformations determined by the images of three roots (Online Resource~4).  The same computation gives $\Z/4$ for $y^2=xK(x^4)$ with $K$ an integer quintic, a member of $\mathcal S_{11,2}$, and $\Z/8$ for $y^2=xK(x^8)$ with $K$ an integer cubic, a member of $\mathcal S_{13,4}$.
\end{exa}

\begin{rem}
By \cref{thm:chiral-odd}(ii), the locus $\mathcal A_n$ is cut out, on forms with simple roots, by the invariants of odd degree of binary forms of degree $2n$; an explicit equation such as~\cref{eq:R2F2} below requires generators of that invariant ring, which are available only in low degree.  The stratum $\mathcal S_{n,k}$ itself is described by the coefficients of $K$ modulo the torus $x\mapsto\lambda x$ and the reversal $K(u)\mapsto u^dK(1/u)$, the analogue of the dihedral invariants of~\cite{2003-3} for $\mathcal L_g$.
\end{rem}

\subsection{Degree three}
\label{sec:degree-three}

\bmhead{The invariants of the sextic}
By \cref{cor:maxwell}, a nonzero real harmonic of degree three corresponds to a sextic $F=\lambda Q_aQ_bQ_c$ with $a,b,c\in\R^3\setminus0$ and $\lambda\in\R^\times$.  The invariant ring of binary sextics is generated by the Igusa--Clebsch invariants $I_2,I_4,I_6,I_{10}$ and the skew invariant $R$ of degree $15$, with $R^2\in\C[I_2,I_4,I_6,I_{10}]$, see~\cite{Clebsch1872,Igusa1960,2004-3}; $I_{10}$ is the discriminant, so $I_{10}\ne0$ iff the axes $\R a,\R b,\R c$ are distinct, and then $y^2=F(z,1)$ is a genus-two curve whose isomorphism class is the point $[J_2:J_4:J_6:J_{10}]$ of $\mathcal M_2=\{J_{10}\ne0\}\subset\PP_{(2,4,6,10)}$ in Igusa's invariants $J_2,J_4,J_6,J_{10}$, see~\cite{Igusa1960,2019-1,2024-03}.  For $R$ we use the transvectant construction: in the normalization~\cref{eq:transvectant} put
\[
\Theta=(F,F)_4,\qquad Y_1=(F,\Theta)_4,\qquad Y_2=(\Theta,Y_1)_2,\qquad Y_3=(\Theta,Y_2)_2,
\]
so that $\Theta$ is a quartic and $Y_1,Y_2,Y_3$ are quadratic covariants, write $Y_r=y_{r0}z_0^2+2y_{r1}z_0z_1+y_{r2}z_1^2$, and let $R=\det(y_{rs})$.

\bmhead{The locus $\mathcal L_2$ and the square of the skew invariant}
Let $\mathcal L_2\subset\mathcal M_2$ be the locus of genus-two curves with an \textbf{elliptic involution}, an involution other than the hyperelliptic one, equivalently of curves with a degree-two elliptic subcover.  Its equation $F_2(J_2,J_4,J_6,J_{10})=0$ in $\PP_{(2,4,6,10)}$, with $F_2$ an irreducible weighted-homogeneous polynomial of degree $30$ with $29$ terms, is computed in~\cite{2000-2} from Jacobi's normal form $Y^2=X^6-s_1X^4+s_2X^2-1$, see also~\cite{2003-3,sh-128}.  We normalize $F_2$ as the primitive integer polynomial in which the coefficient of $J_2^6J_6^3$ is $-1$.

The square of the skew invariant is a polynomial in $J_2,J_4,J_6,J_{10}$, and with the normalizations above it is
\begin{equation}
\label{eq:R2F2}
R^2=-\frac{2^{18}}{3^{18}\,5^{20}}\,F_2(J_2,J_4,J_6,J_{10}).
\end{equation}
The constant was determined by exact evaluation (\cref{sec:comp-verified}); the identity itself is proved in \cref{thm:chirality3} below.  We put $\tilde R=2^{-9}\,3^{9}\,5^{10}\,R$, so that $\tilde R^2=-F_2(J_2,J_4,J_6,J_{10})$, and use $\tilde R$ as the coordinate; without the factor the gauge coordinate~\cref{eq:gauge} of $-iR$ is of order $10^{-16}$ on generic chiral band-three harmonics.

By \cref{lem:reality}(iii), $J_2,J_4,J_6,J_{10}$ are real on real harmonics of degree three and $-iR$ is a real pseudo-invariant.  They do not all vanish on a nonzero real harmonic: otherwise $R=0$ by~\cref{eq:R2F2}, so $F$ would be a nullform with a root of multiplicity at least four, impossible for roots in antipodal pairs.  The point $[J_2:J_4:J_6:J_{10}]\in\PP_{(2,4,6,10)}(\R)$ determines the harmonic up to rotation, scale and reflection, since equal even invariants give $R'=\pm R$ by~\cref{eq:R2F2}; adjoining $-i\tilde R$ restores completeness up to rotation and positive scale (\cref{thm:sep-radial}).  On $V_3$ the orientation of band three relative to the lower bands is recorded only by joint invariants.

\bmhead{The theorem in degree three}
In degree three \cref{thm:chiral-odd} is explicit: the odd part of the invariant ring of sextics is $R\cdot\C[J_2,J_4,J_6,J_{10}]$, so condition (ii) reads $R(F)=0$, and condition (iv) reduces to two configurations of the three axes.  Recall that a nonzero $f\in\Harm_3$ has three Maxwell axes $\R a,\R b,\R c$, unique with multiplicity, with $f\equiv\lambda\ip ax\ip bx\ip cx\pmod q$ by \cref{lem:reality}(ii) and $F_f=\lambda Q_aQ_bQ_c$ by \cref{cor:maxwell}.  For coplanar axes, the half-turn about the normal $n$ of their plane reverses $f$.  For the bisector type, $n$ bisects the angle of $\R a$ and $\R b$, and $\R c$ is orthogonal to $n$, so the half-turn about $n$ exchanges $\R a$ and $\R b$ and reverses $\R c$.  In both cases $f$ is achiral and $R(F_f)=0$.

\begin{thm}
\label{thm:chirality3}
Let $0\ne f\in\Harm_3$ with axes $a,b,c$ and sextic $F=F_f$.  The following are equivalent.
\begin{enumerate}[label=(\roman*)]
\item $f$ is achiral: its mirror image lies in $\SO(3)\cdot f$.
\item $R(F)=0$.
\item There is a half-turn $\rho\in\SO(3)$ with $\rho\cdot f=-f$.
\item Either the three axes are coplanar, or one axis is orthogonal to an angle bisector of the other two.
\end{enumerate}
Moreover, when $J_{10}(F)\ne0$, these are equivalent to the genus-two curve $y^2=F(z,1)$ having an elliptic involution, that is, lying in the locus $\mathcal L_2$; the classical form of this statement is due to Clebsch~\cite{Clebsch1872,2000-2}, and its equivalence with (i) is the case $n=3$ of \cref{prop:n35}.  The identity~\cref{eq:R2F2} holds for every complex sextic, and for every nonzero $f\in\Harm_3$, including those with $J_{10}(F)=0$, the conditions (i)--(iv) are equivalent to
\begin{enumerate}[label=(\roman*),start=5]
\item $F_2\bigl(J_2(F),J_4(F),J_6(F),J_{10}(F)\bigr)=0$, with $F_2$ the equation of $\mathcal L_2$ normalized as in \cref{sec:bands}.
\end{enumerate}  Finally, as a polynomial in $(\lambda,a,b,c)$,
\begin{equation}
\label{eq:Rfactor}
-iR(F)=\lambda^{15}\det(a,b,c)\,G\bigl(\ip aa,\ip ab,\ip ac,\ip bb,\ip bc,\ip cc\bigr)
\end{equation}
for a polynomial $G$ with rational coefficients, and on the noncoplanar configurations $G$ vanishes exactly on those of bisector type.
\end{thm}

\begin{proof}
(i)$\Leftrightarrow$(ii) is \cref{thm:chiral-odd}(ii), since the odd part of the invariant ring of sextics is $R\cdot\C[J_2,J_4,J_6,J_{10}]$.

(i)$\Rightarrow$(iv).  Achirality means some $g\in\Or(3)\setminus\SO(3)$ satisfies $g\cdot f=f$.  Write $g=-\rho$ with $\rho\in\SO(3)$; since $-1$ acts on $\Harm_3$ by $-1$, this says $\rho\cdot f=-f$.  Since $\rho$ preserves $q$, it acts on the Maxwell representation: $\rho\cdot f\equiv\lambda\prod_k\ip{\rho a_k}{x}\bmod q$, so by uniqueness of the axes $\rho$ permutes the multiset $\{\R a,\R b,\R c\}$, $\rho a_k=t_ka_{\pi(k)}$, and comparing scalars $\prod_kt_k=-1$ (the product $\prod t_k$ does not depend on the choice of representatives, as rescaling $a_k\mapsto s_ka_k$ multiplies it by $\prod s_k/s_{\pi(k)}=1$).  Let $\rho$ be rotation by $\theta\in(0,2\pi)$ about the unit vector $n$.  A rotation fixes a line $\ell\ne\R n$ only if $\theta=\pi$ and $\ell\perp n$; it fixes $\R n$ with $t=+1$, and a line $\ell\perp n$ under a half-turn with $t=-1$.

\textbf{Case 1: $\pi$ is the identity.}  Each axis is either $\R n$ ($t=+1$) or, if $\theta=\pi$, perpendicular to $n$ ($t=-1$).  Since $\prod t_k=-1$, either one axis is $\perp n$ and two coincide with $\R n$, or all three are $\perp n$.  In both cases the axes are coplanar.

\textbf{Case 2: $\pi$ is a transposition}, say $\R a\ne\R b$ are exchanged and $\R c$ is fixed.  If $\theta=\pi$, then $\rho^2=1$ gives $t_at_b=1$, so $t_c=-1$ and $c\perp n$; moreover $\rho a=2\ip an n-a\in\R b$ forces $n\in\mathrm{span}(a,b)$, and a half-turn about a line in the plane of $\R a,\R b$ exchanges them iff that line bisects their angle.  If $\theta\ne\pi$, then $\rho^2$ is a nontrivial rotation fixing the lines $\R a,\R b,\R c$; the lines $\R a,\R b$ are not $\R n$ (else $\rho$ would fix them), so $\rho^2$ is the half-turn about $n$ with $a,b\perp n$, i.e.\ $\theta=\pm\pi/2$; then $\rho a\perp a$ gives $a\perp b$, and $c$, fixed by a quarter-turn, lies on $\R n$.  So $a,b,c$ are mutually orthogonal, and $c$ is orthogonal to both bisectors of $\R a,\R b$.  In either subcase (iv) holds.

\textbf{Case 3: $\pi$ is a 3-cycle}, so the three lines are distinct.  Choose representatives $b=\rho a$, $c=\rho b$; then $\rho c=\rho^3a$ and $\rho^3$ fixes $\R a$.  Since $\R a\ne\R n$ (else $\rho$ would fix $\R a$), either $\rho^3=1$, in which case $\rho c=a$, all $t_k=1$ and $\prod t_k=1$, a contradiction; or $\rho^3$ is the half-turn about $n$ with $a\perp n$, so $\theta\in\{\pi/3,5\pi/3\}$ ($\theta=\pi$ would fix $\R a$) and $a,b,c$ all lie in the plane $\perp n$.  The axes are coplanar.

(iv)$\Rightarrow$(iii).  If the axes are coplanar with normal $n$, the half-turn about $n$ sends each $a_k\mapsto-a_k$, so $\rho\cdot f=(-1)^3f=-f$.  If $c\perp n$ with $n$ a bisector of $\R a,\R b$, choose unit representatives with $n\propto\hat a+\hat b$; the half-turn about $n$ maps $\hat a\mapsto2\ip{\hat a}n n-\hat a=\hat b$, $\hat b\mapsto\hat a$, and $c\mapsto-c$, so $\rho\cdot f=-f$.

(iii)$\Rightarrow$(i).  If $\rho\cdot f=-f$ then $(-\rho)\cdot f=f$ and $-\rho\in\Or(3)\setminus\SO(3)$.

For the complex statement, let $J_{10}(F)\ne0$.  Then $y^2=F(z,1)$ is a smooth genus-two curve, and $R(F)=0$ iff the curve has an elliptic involution, see~\cite{2000-2}.  Combine with (i)$\Leftrightarrow$(ii).  Note that (iii) supplies, for real $F$, an involution defined over $\R$ and in $\SO(3)$.

For~\cref{eq:R2F2}: the invariants $J_2,J_4,J_6,J_{10}$ have weighted degrees $2,4,6,10$, and $F_2$ is weighted-homogeneous of degree $30$ and irreducible, see~\cite{2000-2}.  The polynomial $R^2$ is an even invariant, so it lies in $\C[J_2,J_4,J_6,J_{10}]$, and it is weighted-homogeneous of degree $30$ as well.  On the open set $J_{10}\ne0$ the zero set of $R$ is $\mathcal L_2$ by the criterion of~\cite{Clebsch1872,2000-2}, so $R^2$ vanishes on $\mathcal L_2$.  The locus $\mathcal L_2$ is dense in the irreducible surface $V(F_2)$, and $(F_2)$ is a prime ideal, so $R^2\in(F_2)$ and $R^2=cF_2$ with $c$ weighted-homogeneous of degree zero, that is, a constant.  A single exact evaluation at a sextic with $F_2\ne0$ determines $c$; it gives $c=-2^{18}3^{-18}5^{-20}$ (\cref{sec:comp-verified}).  Hence (v) is equivalent to (ii) for every $F$, including those with $J_{10}=0$.

For~\cref{eq:Rfactor}: the coefficients of $F$ are $\lambda$ times trilinear expressions in $(a,b,c)$, so $-iR(F)$ is $\lambda^{15}$ times a polynomial $P(a,b,c)$ of tridegree $(15,15,15)$, with real coefficients by \cref{lem:reality}(iii), and $P$ is $\SO(3)$-invariant under the diagonal action because $\Psi_3$ is equivariant and rotating $f$ rotates its axes.  By the first fundamental theorem for $\SO(3)$ (see~\cite{Weyl46}), $P\in\R[\ip{a_i}{a_j},\det(a,b,c)]$, and since $\det(a,b,c)^2$ is a polynomial in inner products, $P=G_0+\det(a,b,c)G_1$ with $G_0,G_1$ polynomials in inner products.  Each inner product has total degree two and $\det$ has total degree three; the total degree of $P$ is $45$, so every monomial contains an odd number of determinant factors, hence $G_0=0$.  Finally $\det(a,b,c)=0$ is the coplanar case, and by the equivalences $G$ vanishes on a noncoplanar configuration iff it is of bisector type.
\end{proof}

\begin{rem}
Condition (v) is a practical test: $F_2$ is a real polynomial of degree $30$ in the $\Or(3)$-invariants $J_2,J_4,J_6,J_{10}$, so achirality in degree three can be decided without evaluating the degree-$15$ pseudo-invariant $R$ and without complex arithmetic.
\end{rem}

\begin{rem}
\label{rem:mirror}
For $l=3$, \cref{thm:chirality3}(iii) says that $f$ is achiral iff it is fixed by a reflection: if $\rho$ is the half-turn about $\R n$ with $\rho\cdot f=-f$, then $-\rho$ is the reflection in $n^\perp$ and $(-\rho)\cdot f=f$.  For general $l$, achirality means $S\cdot f=f$ for some $S\in\Or(3)\setminus\SO(3)$, and $S$ is a reflection only when $-S$ is a half-turn; \cref{thm:chiral-odd}(iv) describes the general case.
\end{rem}
\section{Metrics and Stability}
\label{sec:metrics}

This section defines the distance on the quotient $V_L/\SO(3)$ and proves the two directions of stability: small changes of a tuple change the invariants little, and small changes of the invariants come only from small changes of the tuple modulo rotation.

\subsection{The orbit distance}

The group $\SO(3)$ acts on $V_L$ linearly and isometrically for the $L^2(d\omega)$ norm, since $D_l(Q)$ is orthogonal.  The \textbf{orbit distance} of $c,c'\in V_L$ is the distance between their orbits,
\begin{equation}
\label{eq:orbitdist}
\delta(c,c')=\min_{R\in\SO(3)}\norm{c-R\cdot c'}.
\end{equation}
The minimum exists because $\SO(3)$ is compact.

The minimization over $\SO(3)$ is nonconvex and needs documented multistart.  A scale-free version replaces each nonzero input by its unit-norm representative before minimizing.

The \textbf{invariant-coordinate distance} is any distance between normalized invariant vectors~\cref{eq:gauge}.  On a chart where a fixed set of coordinates is nonzero, the root-degree logarithms $w_k^{-1}\log\abs{x_k}$ turn scaling into translation and give a scale-free chart distance.  The orbit distance is canonical once representation, truncation and inner product are fixed; an invariant-coordinate distance depends on the embedding $\mathcal I$.  The following theorem relates the two.

\subsection{Stability in both directions}

\begin{thm}
\label{thm:stability}
The following are true:
\begin{enumerate}[label=(\roman*)]
\item $\delta$ is a metric on $V_L/\SO(3)$.
\item Let $I$ be a homogeneous polynomial invariant of degree $w$ and $B$ the unit ball of $V_L$.  Then
\[
\abs{I(c)-I(c')}\le\bigl(\sup_B\abs{\nabla I}\bigr)\,\delta(c,c')\qquad(c,c'\in B);
\]
in particular every invariant is Lipschitz on bounded sets with respect to the orbit distance, and so is every normalized coordinate~\cref{eq:gauge} on any set where $H_w$ is bounded below.

\item Let $\mathcal I$ be a separating polynomial invariant map and $\mathcal K\subset V_L$ compact.  There are $C_{\mathcal K}>0$ and $\alpha_{\mathcal K}\in(0,1]$ such that
\[
\delta(c,c')\le C_{\mathcal K}\,\abs{\mathcal I(c)-\mathcal I(c')}^{\alpha_{\mathcal K}}\qquad(c,c'\in\mathcal K).
\]
\item At a point of the principal stratum where the Jacobian of $\mathcal I$ restricted to a slice has full rank, the exponent in (iii) can be taken to be $1$ locally: $\mathcal I$ is a local bi-Lipschitz embedding of the quotient into $\R^N$.  Here a \textbf{slice} at $c$ is a submanifold through $c$, transverse to the orbit of $c$ and of complementary dimension, which parametrizes the quotient near the orbit of $c$.
\end{enumerate}
\end{thm}

\begin{proof}
(i) $\delta$ is the quotient distance of an isometric action of a compact group, whose orbits are closed.  (ii) If $R$ attains $\delta(c,c')$, then $I(c')=I(R\cdot c')$ and $R\cdot c'\in B$, and the mean value theorem on the segment from $c$ to $R\cdot c'$ gives the bound; the map $x\mapsto\hat x$ of~\cref{eq:gauge} is Lipschitz where $H_w$ is bounded below.  (iii) On $B'\times B'$, with $B'\supset\mathcal K$ a closed ball, $\delta$ and $\abs{\mathcal I(c)-\mathcal I(c')}$ are continuous semialgebraic functions, the first by the Tarski--Seidenberg theorem, and they have the same zero set by separation; the \L{}ojasiewicz inequality (see~\cite{BochnakCosteRoy98}) gives the bound, and $\alpha$ may be decreased to lie in $(0,1]$.  (iv) On the principal stratum the stabilizer is locally constant, so by the slice theorem (see~\cite{Bredon72}) a slice $\Sigma$ at $c$ is a chart of the quotient on which $\delta$ is bi-Lipschitz equivalent to the Euclidean distance; if $\mathcal I|_\Sigma$ has full rank at $c$, it is bi-Lipschitz near $c$, and by invariance the bound holds on an invariant neighbourhood of the orbit.
\end{proof}

By (ii) and (iii), $\mathcal I$ induces a homeomorphism from the image of $\mathcal K$ in $V_L/\SO(3)$ onto its image which is Lipschitz and has a H\"older continuous inverse.

The theorem holds for any compact group acting linearly.  What is specific to the present situation is the size of the exponent: it cannot be taken to be $1$, and chirality lowers it further.

\begin{exa}
\label{exa:holder}
Let $V=\Harm_1$ with $\mathcal I(a)=\abs a^2$, and let $\mathcal K$ be the unit ball.  Here $\delta(a,0)=\norm{\ip a\cdot}=\abs a/\sqrt3$, while $\abs{\mathcal I(a)-\mathcal I(0)}=\abs a^2$.  Hence $\delta(a,0)=3^{-1/2}\abs{\mathcal I(a)-\mathcal I(0)}^{1/2}$, and no inequality $\delta\le C\abs{\Delta\mathcal I}^{\alpha}$ with $\alpha>1/2$ can hold near the origin.  The same happens in every $V_L$: all invariants other than $f_0$ have degree at least two and vanish to second order at $c=0$, while $\delta(c,0)=\norm c$ vanishes to first order.  So $\alpha_{\mathcal K}\le1/2$ whenever $\mathcal K$ contains a neighborhood of a point $c_0$ whose stabilizer is not the principal one: every invariant is stationary at $c_0$ in the directions of the slice on which the stabilizer acts without fixed vectors, while the orbit distance is of first order there; and $\alpha_{\mathcal K}$ can be smaller at points where higher-degree invariants are needed to separate: for a chiral pair $c=(a,Q)\in V_2$ and its mirror $c^*=(-a,Q)$, $\delta(tc,tc^*)=t\,\delta(c,c^*)$, while the seven invariants of \cref{thm:L2} differ only in $\varepsilon$, by $2t^6\abs{\varepsilon(c)}$, so $\alpha_{\mathcal K}\le1/6$ on any neighborhood of the origin.
\end{exa}

The exponent is bounded stratum by stratum.  At a point $c_0$ with stabilizer $H$, the slice theorem identifies a neighbourhood of the orbit in $V_L/\SO(3)$ with a neighbourhood of $0$ in $N/H$, $N$ a slice, with comparable metrics (proof of \cref{thm:stability}(iv)).  For $V_2$ and the seven invariants of \cref{thm:L2} the strata are read off from the eigenstructure of $Q$ and the position of $a$ in an eigenbasis, and the following bounds hold.

\begin{prop}
\label{prop:exponents}
Let $\mathcal I$ be the seven invariants of \cref{thm:L2} on $V_2$, and write $c=(f_0,a,Q)$ with $Q=\diag(\lambda_1,\lambda_2,\lambda_3)$ in an eigenbasis.
\begin{enumerate}[label=(\roman*)]
\item If the $\lambda_i$ are distinct and at most one coordinate of $a$ vanishes, then $c$ lies on the principal stratum and the Jacobian of $\mathcal I$ on a slice at $c$ has rank $6$; so $\mathcal I$ is locally bi-Lipschitz there.
\item If the $\lambda_i$ are distinct and exactly two coordinates of $a$ vanish, the stabilizer is $\Z/2$ and $\alpha_{\mathcal K}\le1/2$ on every neighbourhood of $c$.
\item If the $\lambda_i$ are distinct and $a=0$, the stabilizer is the Klein group $K$ and $\alpha_{\mathcal K}\le1/3$ on every neighbourhood of $c$.
\item On every neighbourhood of the origin, $\alpha_{\mathcal K}\le1/6$.
\end{enumerate}
\end{prop}

\begin{proof}
The stabilizer of $c$ consists of the rotations commuting with $Q$ and fixing $a$; for distinct $\lambda_i$ these are the elements $\diag(\epsilon_1,\epsilon_2,\epsilon_3)$ of $K$ with $\epsilon_ia_i=a_i$, which gives the stabilizers in (i)--(iii).

(i) A slice at $c$ is $\{(f,b,\diag(\mu_1,\mu_2,\mu_3)):\sum\mu_i=0\}$, of dimension $6$.  The invariant $f_0$ has nonzero derivative only along $f$.  The map $(\lambda_1,\lambda_2)\mapsto(\tr Q^2,\det Q)$ on traceless diagonal matrices has Jacobian determinant proportional to $\prod_{i<j}(\lambda_i-\lambda_j)\ne0$, so $\tr Q^2,\det Q$ have rank $2$ along $\mu$.  The three invariants $\abs a^2,a^{\top}Qa,a^{\top}Q^2a$ equal $\sum_ia_i^2\lambda_i^k$, $k=0,1,2$, and their derivatives along $b$ form the matrix $(2a_i\lambda_i^k)_{k,i}$, a Vandermonde matrix times $\diag(2a_i)$; its rank is the number of nonzero $a_i$.  If no $a_i$ vanishes this is $3$.  If exactly $a_i=0$, it is $2$, and $\varepsilon=a_1a_2a_3\prod_{j<k}(\lambda_k-\lambda_j)$ has derivative along $b_i$ equal to $\prod_{j\ne i}a_j\prod_{j<k}(\lambda_k-\lambda_j)\ne0$ while the other three have derivative $0$ along $b_i$, so the rank along $b$ is again $3$.  The Jacobian is block triangular with respect to $(f,\mu,b)$, hence of rank $6$, and \cref{thm:stability}(iv) applies.

(ii) Let $a=(a_1,0,0)$ with $a_1\ne0$, and for small $t$ let $c_t=(f_0,(a_1,tb_2,tb_3),Q)$ and $c'_t=(f_0,(a_1,tb_2,-tb_3),Q)$ with $b_2b_3\ne0$.  The six $\Or(3)$-invariants are functions of $a_i^2$ and $Q$, so they agree on $c_t$ and $c'_t$, and $\varepsilon(c_t)-\varepsilon(c'_t)=2t^2a_1b_2b_3\prod_{j<k}(\lambda_k-\lambda_j)$.  A rotation $R$ with $\abs{c_t-R\cdot c'_t}$ small must be close to the stabilizer $\{1,\diag(1,-1,-1)\}$ of $Q$ and $a$, and neither element sends $(a_1,tb_2,-tb_3)$ to $(a_1,tb_2,tb_3)$; hence $\delta(c_t,c'_t)\ge c_1t$ for small $t$, while $\abs{\mathcal I(c_t)-\mathcal I(c'_t)}\le c_2t^2$.

(iii) Let $c_t=(f_0,tb,Q)$ and $c'_t=(f_0,tb',Q)$ with $b=(b_1,b_2,b_3)$, all $b_i\ne0$, and $b'=(b_1,b_2,-b_3)$.  Again the $\Or(3)$-invariants agree, $\varepsilon(c_t)-\varepsilon(c'_t)=2t^3b_1b_2b_3\prod_{j<k}(\lambda_k-\lambda_j)$, and $b'\notin K\cdot b$, so $\delta(c_t,c'_t)\ge c_1t$ and $\abs{\mathcal I(c_t)-\mathcal I(c'_t)}\le c_2t^3$.

(iv) is \cref{exa:holder}.
\end{proof}

For $Q$ with a repeated eigenvalue the stabilizer depends on $a$ as well, and these points are not covered by (i)--(iii); at the point with stabilizer $\SO(2)$ used in the benchmark of \cref{sec:bench-exponents}, the tested perturbations show a discrepancy of the invariants quadratic in the orbit distance.

\section{Surfaces with Spherical Topology}
\label{sec:surfaces}

We now apply the preceding results to a smoothly embedded two-sphere $S\subset\R^3$.  The spherical harmonic coefficients of a function on $S$ depend on the parameterization of $S$ by the sphere; a conformal parameterization normalized by its conformal barycenter reduces this dependence to a rotation, and the invariants of \cref{sec:invariants} then give a descriptor of $S$ which is exactly invariant, complete for the truncation, records handedness, and is stable.

A set with a free and transitive action of a group $G$ is a \textbf{$G$-torsor}.  For a probability measure $\mu$ on $\Sph$ and a continuous map $g$ of $\Sph$, the pushforward $g_*\mu(E)=\mu(g^{-1}(E))$ satisfies $\int f\,d(g_*\mu)=\int f\circ g\,d\mu$, and an \textbf{atom} of $\mu$ is a point of positive mass.  Let $\Ball=\{x\in\R^3:\abs x<1\}$ be the open unit ball.  The \textbf{M\"obius group} $\Mob\cong\PSL_2(\C)$ is the group of orientation-preserving conformal diffeomorphisms of $\Sph$, see~\cite{Ahlfors81}.

\subsection{Conformal parameterizations and the conformal barycenter}

Throughout, $S\subset\R^3$ is a smoothly embedded two-sphere, oriented by its outward normal, with induced metric $g_S$, area form $dA$ and area $\abs S$.  A diffeomorphism $\varphi\colon\Sph\to S$ is \textbf{conformal} if $\varphi^*g_S=e^{2\rho}g_{\rm round}$ for a smooth function $\rho$; $e^{2\rho}$ is its \textbf{conformal factor}.  By the uniformization theorem (see~\cite{deSaintGervais16}), the set $\Phi(S)$ of orientation-preserving conformal diffeomorphisms $\Sph\to S$ is nonempty, and $\Mob$ acts on it freely and transitively by $\varphi\mapsto\varphi\circ g$; that is, $\Phi(S)$ is a $\Mob$-torsor.  Every element of $\Mob$ extends uniquely to a M\"obius transformation of $\overline{\Ball}$, the extended group acts transitively on $\Ball$, and the stabilizer of the origin is $\SO(3)$ acting linearly, see~\cite{Ahlfors81}.  Hence removing the three nonrotational degrees of freedom of $\Mob$ amounts to choosing a point of $\Ball$.

For $x\in\Ball$, the \textbf{hyperbolic translation} $\tau_x\colon\overline{\Ball}\to\overline{\Ball}$ is the map
\begin{equation}
\label{eq:tau}
\tau_x(u)=\frac{(1-\abs x^2)(u-x)-\abs{u-x}^2x}{1-2\ip ux+\abs u^2\abs x^2}
\end{equation}
It is the M\"obius transformation of $\overline{\Ball}$ that translates along the hyperbolic geodesic through $0$ and $x$; it sends $x$ to $0$, maps $\Sph$ to itself, and $\tau_0$ is the identity.  For a probability measure $\mu$ on $\Sph$ put $\xi_\mu(x)=\int_{\Sph}\tau_x(u)\,d\mu(u)$, $x\in\Ball$, the Euclidean center of mass of $(\tau_x)_*\mu$.

Douady and Earle~\cite{DouadyEarle86} proved the following.  If $\mu$ has no atoms, then the vector field $\xi_\mu$ has a unique zero $B(\mu)\in\Ball$, and $B(g_*\mu)=g\bigl(B(\mu)\bigr)$ for $g\in\Mob$.  The point $B(\mu)$ is the \textbf{conformal barycenter} of $\mu$.  Since $\tau_0$ is the identity, $B(\mu)=0$ if and only if $\int_{\Sph}u\,d\mu(u)=0$.

Each $\varphi\in\Phi(S)$ determines the probability measure $\mu_\varphi=\abs S^{-1}\varphi^*(dA)$ on $\Sph$.  We call $\varphi$ \textbf{centered} if $\int_{\Sph}u\,d\mu_\varphi(u)=0$, and write $\Phi_0(S)$ for the set of centered parameterizations.

\begin{lem}
\label{lem:centering}
The following are true:
\begin{enumerate}[label=(\roman*)]
\item $\Phi_0(S)$ is nonempty and is a single orbit under $\SO(3)$ acting by precomposition.
\item If $T(x)=\lambda Rx+b$ with $\lambda>0$, $R\in\SO(3)$, then $\Phi_0(T(S))=T\circ\Phi_0(S)$.
\item If $T(x)=-\lambda Rx+b$ and $r$ is any reflection of $\Sph$ in a plane through the origin, then $\Phi_0(T(S))=T\circ\Phi_0(S)\circ r$.
\end{enumerate}
\end{lem}

\begin{proof}
(i) is the theorem of Douady and Earle~\cite{DouadyEarle86} applied to $\mu_\varphi$, which has no atoms: since $\mu_{\varphi\circ g}=(g^{-1})_*\mu_\varphi$, $\varphi\circ g$ is centered iff $g(0)=B(\mu_\varphi)$, and two such $g$ differ by an element of $\SO(3)$, the stabilizer of the origin; see also~\cite{BadenCraneKazhdan18}.  (ii) and (iii) hold because a similarity with ratio $\lambda$ multiplies area by $\lambda^2$, so $\mu_{T\circ\varphi}=\mu_\varphi$, and an orientation-reversing $T$ is compensated by the isometry $r$, with $\mu_{T\circ\varphi\circ r}=r_*\mu_\varphi$.
\end{proof}

\begin{rem}
\label{rem:mesh}
For a triangulated surface the smooth construction is replaced by its discrete counterpart of~\cite{BadenCraneKazhdan18}.  The pulled-back area measure becomes a finite sum of point masses, $\mu_h=\sum_iw_i\delta_{u_i}$.  It has a unique conformal barycenter provided no atom carries mass $\ge1/2$, by the stability criterion of Cantarella and Schumacher~\cite{CantarellaSchumacher22}.  Uniqueness is not conditioning.  The Hessian of the associated energy, computed in the same paper, governs the sensitivity of the barycenter to the mesh (\cref{rem:param-stability}).
\end{rem}

\subsection{The conformal center and the harmonic coefficients}

For $\varphi\in\Phi_0(S)$ put $\bar c(S)=\int_{\Sph}\varphi(u)\dw$.  The right side does not depend on the choice of $\varphi\in\Phi_0(S)$.  Indeed, by \cref{lem:centering}(i) any other choice is $\varphi\circ R$ with $R\in\SO(3)$, and $d\omega$ is $\SO(3)$-invariant.  We call $\bar c(S)$ the \textbf{conformal center} of $S$.

It lies in the convex hull of $S$, differs in general from the area centroid $\abs S^{-1}\int_Sx\,dA=\int_{\Sph}\varphi\,d\mu_\varphi$, and is natural: $\bar c(T(S))=T(\bar c(S))$ for every similarity $T$, by \cref{lem:centering}(ii) and (iii), since $T$ is affine and $d\omega$ is invariant under reflections.

Define the \textbf{centered coordinate map} and the \textbf{radial function} of $\varphi$ by $X_\varphi(u)=\varphi(u)-\bar c(S)$ and $d_\varphi(u)=\abs{X_\varphi(u)}$.  Its band components are $f_l(\varphi)=\pi_ld_\varphi\in\Harm_l$; write $f^{(L)}(\varphi)=(f_0,\dots,f_L)$.  Since $\Harm_0$ consists of the constants, $C_0(\varphi):=\pi_0X_\varphi=\int_{\Sph}X_\varphi\dw$, which vanishes identically by the definition of $\bar c(S)$; this is why the conformal center, rather than the area centroid, is taken as the origin.

\begin{cor}
\label{cor:orbit}
For every $L\ge1$ the following hold.
\begin{enumerate}[label=(\roman*)]
\item The set $\{f^{(L)}(\varphi):\varphi\in\Phi_0(S)\}$ is one orbit of the diagonal action of $\SO(3)$ on $\bigoplus_{l\le L}\Harm_l$.
\item If $T(x)=\lambda Rx+b$ with $\lambda>0$ and $R\in\SO(3)$, then $f_l(T\circ\varphi)=\lambda f_l(\varphi)$.
\item If $T(x)=-\lambda Rx+b$ and $r$ is as in \cref{lem:centering}(iii), then $f_l(T\circ\varphi\circ r)=\lambda f_l(\varphi)\circ r$.
\end{enumerate}
\end{cor}

\begin{proof}
Let $Q\in\SO(3)$.  Since $\bar c(S)$ does not depend on $\varphi$, $X_{\varphi\circ Q}(u)=\varphi(Qu)-\bar c(S)=(Q^{-1}\cdot X_\varphi)(u)$, so $d_{\varphi\circ Q}=Q^{-1}\cdot d_\varphi$ and, since $\pi_l$ commutes with the action, $f_l(\varphi\circ Q)=Q^{-1}\cdot f_l(\varphi)$ with the same $Q$ in every band.  Together with \cref{lem:centering}(i) this proves (i).  For (ii), \cref{lem:centering}(ii) and the naturality of $\bar c$ give $X_{T\circ\varphi}=T\circ\varphi-T(\bar c(S))=\lambda R\,X_\varphi$, hence $d_{T\circ\varphi}=\lambda\,d_\varphi$, and (ii) follows by applying $\pi_l$.  For (iii), \cref{lem:centering}(iii) gives in the same way $X_{T\circ\varphi\circ r}=-\lambda R\,X_\varphi\circ r$, hence $d_{T\circ\varphi\circ r}=\lambda\,d_\varphi\circ r$; since $r$ is orthogonal, $\pi_l$ commutes with composition by $r$, and (iii) follows.
\end{proof}

By \cref{cor:orbit} the entire tuple of coefficients of a centered conformal parameterization is determined by $S$ up to one common rotation of the sphere.  The radial expansion retains only $\abs{X_\varphi}$; no reconstruction theorem from radial data is asserted.

\subsection{Truncation}

Conformal maps of elongated or convoluted surfaces have large area distortion, so a small spherical cap may carry a large part of the surface.  The truncation error in $L^2(d\omega)$ is nevertheless controlled by the area alone; norms without subscript are $L^2(d\omega)$ norms.

\begin{prop}
\label{prop:truncation}
Let $\varphi\in\Phi_0(S)$ with conformal factor $e^{2\rho}$, let $X=X_\varphi$ and $X_l=\pi_lX$, and let $X^{(L)}=\sum_{l\le L}X_l$ be the degree-$L$ truncation of $X$.  Then
\begin{equation}
\label{eq:trunc}
\norm{X-X^{(L)}}_{L^2(d\omega)}^2 \le\frac{\abs S}{2\pi\,(L+1)(L+2)},\quad
\int_{\Sph}\abs{X-X^{(L)}}^2\,d\mu_\varphi \le\frac{2\,\sup e^{2\rho}}{(L+1)(L+2)}.
\end{equation}
The same bounds hold for $d_\varphi$ in place of $X_\varphi$.
\end{prop}

\begin{proof}
For $u\in\Sph$ let $e_1,e_2$ be an orthonormal basis of $T_u\Sph$ for $g_{\rm round}$.  The Hilbert--Schmidt norm of the differential is $\abs{d\varphi_u}^2=\abs{d\varphi_u(e_1)}^2+\abs{d\varphi_u(e_2)}^2$.  For a conformal map both terms equal $e^{2\rho(u)}$, so $\abs{d\varphi}^2=2e^{2\rho}$.  Since $\varphi^*dA=e^{2\rho}dA_{\rm round}$, we get $\int_{\Sph}\abs{d\varphi}^2dA_{\rm round}=2\abs S$.  Since $X$ differs from $\varphi$ by a constant, $\abs{\nabla X}=\abs{d\varphi}$.  With $d\omega=dA_{\rm round}/4\pi$ this gives $\norm{\nabla X}^2=\abs S/2\pi$.

Since $-\Delta_{\Sph}$ acts on $\Harm_l$ by $l(l+1)$, Green's formula and orthogonality give $\norm{\nabla X}^2=\sum_{l\ge1}l(l+1)\norm{X_l}^2$.  Hence
\[
\norm{X-X^{(L)}}^2=\sum_{l>L}\norm{X_l}^2 \le\frac1{(L+1)(L+2)}\sum_{l>L}l(l+1)\norm{X_l}^2 \le\frac{\norm{\nabla X}^2}{(L+1)(L+2)},
\]
which is the first bound.  For the second bound, $d\mu_\varphi=\abs S^{-1}e^{2\rho}dA_{\rm round}=(4\pi/\abs S)\,e^{2\rho}\,d\omega$, so the area-weighted error is at most $4\pi\sup e^{2\rho}/\abs S$ times the first bound, which equals $2\sup e^{2\rho}/((L+1)(L+2))$.  Finally, $d_\varphi=\abs{X_\varphi}$ is Lipschitz and $\abs{\nabla d_\varphi}\le\abs{\nabla X_\varphi}$ almost everywhere, so the same argument applies to $d_\varphi$.
\end{proof}

The first bound is uniform in the shape.  The second bound shows where a convoluted surface pays.  The geometric error is governed by the largest conformal factor, and an implementation should report this quantity rather than the number of coefficients.  The quantity $\sup e^{2\rho}$ does not depend on the choice of $\varphi\in\Phi_0(S)$, since replacing $\varphi$ by $\varphi\circ R$ replaces $\rho$ by $\rho\circ R$.  Area-preserving parameterizations avoid the second bound.  However, their residual gauge is the infinite-dimensional group of area-preserving diffeomorphisms of $\Sph$, so they cannot support the exact quotient constructed here.

\begin{rem}
\label{rem:param-stability}
\Cref{thm:stability} concerns the passage from coefficients to invariants.  The passage from surface to coefficients is also stable: the uniformizing map, taken modulo $\Mob$, depends continuously on the metric in $C^{k,\alpha}$, so that after centering the parameterization depends continuously on the metric up to $\SO(3)$, and the conformal barycenter is a smooth function of the measure on the set of stable measures, with derivative controlled by the Hessian computed in~\cite{CantarellaSchumacher22}.  A quantitative statement in terms of mesh perturbations is left open.
\end{rem}

\subsection{The certified descriptor}

We collect what has been proved into a single statement about surfaces.  Fix $L\ge1$ and a multihomogeneous generating set of $\mathcal R_L$ with phased real forms $I_1,\dots,I_N$ of total degrees $w_1,\dots,w_N$, as in \cref{thm:sep-radial}, with the $\Or(3)$-invariants listed first and the pseudo-invariants last.  Define
\[
\mathcal D_L(S)=\mathcal I\bigl(f^{(L)}(\varphi)\bigr)\in\R^N,\qquad\varphi\in\Phi_0(S),
\]
and, when $f^{(L)}(\varphi)\ne0$, its class $[\mathcal D_L(S)]\in\PP^+_w(\R)$.

\begin{thm}
\label{thm:certified}
The following are true:
\begin{enumerate}[label=(\alph*)]
\item \textbf{Well defined.}  $\mathcal D_L(S)$ does not depend on the choice of $\varphi\in\Phi_0(S)$.
\item \textbf{Equivariance.}  If $T(x)=\lambda Rx+b$ with $\lambda>0$ and $R\in\SO(3)$, then $\mathcal D_L(T(S))_k=\lambda^{w_k}\mathcal D_L(S)_k$ for every $k$, so $[\mathcal D_L(T(S))]=[\mathcal D_L(S)]$.  If $T$ is an orientation-reversing similarity, the same holds for the $\Or(3)$-invariant coordinates, while each pseudo-invariant coordinate is multiplied by $-\lambda^{w_k}$.
\item \textbf{Completeness.}  $\mathcal D_L(S)=\mathcal D_L(S')$ if and only if the truncated radial expansions of $S$ and $S'$ differ by a rotation of the sphere: $f^{(L)}(\varphi')=Q\cdot f^{(L)}(\varphi)$ for some $\varphi\in\Phi_0(S)$, $\varphi'\in\Phi_0(S')$ and $Q\in\SO(3)$.  Likewise $[\mathcal D_L(S)]=[\mathcal D_L(S')]$ if and only if the truncated expansions differ by a rotation and a positive scaling.
\item \textbf{Chirality.}  $S$ and its mirror image have radial descriptors in the same $\SO(3)$-orbit if and only if every pseudo-invariant coordinate of $\mathcal D_L(S)$ vanishes.
\item \textbf{Stability.}  On every bounded set of descriptors, $\mathcal D_L$ is Lipschitz with respect to the orbit distance, and on every compact set the orbit distance is bounded by a H\"older power of $\abs{\mathcal D_L(S)-\mathcal D_L(S')}$.
\end{enumerate}
\end{thm}

\begin{proof}
(a) By \cref{cor:orbit}(i) the tuples $f^{(L)}(\varphi)$, $\varphi\in\Phi_0(S)$, form one $\SO(3)$-orbit, and $\mathcal I$ is $\SO(3)$-invariant.  (b) By \cref{cor:orbit}(ii), $f^{(L)}(T\circ\varphi)=\lambda f^{(L)}(\varphi)$, and $I_k$ is homogeneous of degree $w_k$.  For an orientation-reversing $T$, \cref{cor:orbit}(iii) gives $f_l(T\circ\varphi\circ r)=\lambda f_l(\varphi)\circ r$, and $f_l\circ r$ lies in the $\SO(3)$-orbit of $(-1)^lf_l$ since $r=-R_0$ for a rotation $R_0$; the effect on $I_k$ is $(-1)^{\sum ld_l}\lambda^{w_k}$ by \cref{lem:reality}(iv).  (c) is \cref{thm:sep-radial} together with (a), and the scale-free statement is the completeness argument of \cref{sec:scale}.  (d) By \cref{cor:orbit}(iii) the descriptor of the mirror image of $S$ lies in the $\SO(3)$-orbit of $c^*$, so (d) is \cref{prop:chirality}.  (e) is \cref{thm:stability}(ii) and (iii), applied to the descriptors $f^{(L)}(\varphi)$ and $f^{(L)}(\varphi')$.
\end{proof}

Every nuisance freedom of a spherical harmonic description of a surface has been removed exactly: the parameterization by (a), the position and scale by (b) and (c), and the handedness is not discarded but recorded by (d).  Nothing has been lost at the level of the truncated coefficients, since (c) says the descriptor is complete for the truncation; the radial reduction and the truncation themselves may lose information about $S$.  And the result is not only correct on exact data but stable with respect to the coefficients of measured data, by (e).  What remains is the numerical realization of each step on triangulated surfaces, which \cref{sec:computation} addresses.
\section{Computation and Verified Algebra}
\label{sec:computation}

The constructions of the previous sections are exact.  This section says how each is carried out on a triangulated surface, which quantities must be reported so that a computed descriptor can be trusted, which algebraic identities were verified independently of any mesh, and how the separation and stability statements behave on harmonic coefficients.

\subsection{From a mesh to invariants}
\label{sec:comp-param}

The input is a triangulated surface $S_h$ with vertices $p_1,\dots,p_n\in\R^3$.  It must be a connected closed combinatorial manifold with $n-e+t=2$, free of degenerate faces, and consistently oriented; its realization in $\R^3$ is assumed free of self-intersections, so that $S_h$ approximates a smoothly embedded two-sphere.  The first step is a discrete conformal map $\psi_h\colon S_h\to\Sph$, $p_i\mapsto u_i$, without flipped triangles, obtained by minimizing a discrete conformal energy on the sphere (see~\cite{GuWangChanThompsonYau04,ChoiLamLui15}); its residual freedom is, up to discretization error, a M\"obius transformation.

The second step is the discrete form of \cref{lem:centering}.  The pulled-back area measure is $\mu_h=\sum_iw_i\delta_{u_i}$ with $w_i$ the fraction of the area of $S_h$ attributed to $p_i$.  By \cref{rem:mesh} it has a unique conformal barycenter $x_h$ provided $\max_iw_i<1/2$, which must be checked.  The point $x_h$ is the zero of $\xi_h(x)=\sum_iw_i\tau_x(u_i)$; it is computed by Newton's method with a line search, started at $x=0$ and monitored for convergence, or by the fixed-point iteration of~\cite{BadenCraneKazhdan18}.  The Jacobian of $\xi_h$, computed in~\cite{CantarellaSchumacher22}, governs the convergence rate and the sensitivity of $x_h$ to the mesh.  Each $u_i$ is then replaced by $\tau_{x_h}(u_i)$, the residual $\abs{\sum_iw_iu_i}$ is recorded, and the remaining rotation of $\Sph$ is left to the invariants.

\label{sec:comp-coef}
The coefficients are integrals against $d\omega$, not against the area of $S_h$, which a conformal map concentrates in some regions of the sphere.  Let the round weight $\omega_i$ be one third of the area of the spherical triangles incident to $u_i$, divided by $4\pi$.  The conformal center $\bar c(S)$ is approximated by $\bar c_h=\sum_i\omega_ip_i$, and $X_i=p_i-\bar c_h$, $d_i=\abs{X_i}$.  With $\mathsf Y$ the $n\times(L+1)^2$ matrix of basis values $Y_{lm}(u_i)$ and $\Omega=\diag(\omega_i)$, the sample vector $s=(d_i)$ has coefficients
\[
c=(\mathsf Y^{\top}\Omega\mathsf Y)^{-1}\mathsf Y^{\top}\Omega\,s.
\]  The matrix $\mathsf Y^{\top}\Omega\mathsf Y$ approximates the identity, and its deviation, together with $\norm{C_0}$, which vanishes in exact arithmetic, measures the quadrature.  The truncation degree is chosen from the residual, compared with the bound $\abs{S_h}/(2\pi(L+1)(L+2))$ of \cref{prop:truncation}; the largest ratio of a face area on $S_h$ to the area of its image, the discrete $\sup e^{2\rho}$, is reported with it.

\label{sec:comp-inv}
The binomial coefficients of $F_{f_l}$ are computed from the $c_{lm}$ by \cref{thm:correspondence}(iii).  As an independent check, $f_l$ is written as a solid harmonic polynomial in $x$ and substituted into~\cref{eq:veronese}; coefficients from a software library must be converted with the library's own phase and normalization at this step.  Invariants have degrees up to $15$, so each is evaluated on $c/\norm c$ and rescaled, $I_k(c)=\norm c^{w_k}I_k(c/\norm c)$, with $\norm c$ kept as the size variable.  The transvectant sequence for $R$ (\cref{sec:bands}) and for the joint invariants~\cref{eq:bispectral} is implemented once, verified in exact rational arithmetic (\cref{sec:comp-verified}), and then run in floating point; the spurious imaginary parts of $J_2,\dots,J_{10}$ and real part of $R$ measure the floating-point error, by \cref{lem:reality}.  The frame coordinates of \cref{rem:frames} are used only when the relative spectral gaps of $Q$ exceed a stated threshold; since the frame has no continuous extension, the threshold only decides where its discontinuity is placed.

\subsection{Diagnostics and verified identities}
\label{sec:comp-diag}

The invariant vector is normalized by the gauge~\cref{eq:gauge}, and the quantities $\abs{\hat x_k}^{1/w_k}\in[0,1]$ are reported; a value near zero flags proximity to the divisor $I_k=0$.  Each of the following diagnostics corresponds to a hypothesis or a bound of the theory.
\begin{enumerate}[label=(\roman*),leftmargin=2em]
\item $\max_iw_i<1/2$ (\cref{rem:mesh}) and the centering residual.
\item The deviation of $\mathsf Y^{\top}\Omega\mathsf Y$ from the identity, and $\norm{C_0}$.
\item The truncation residual, its ratio to the bound of \cref{prop:truncation}, and $\sup e^{2\rho}$.
\item The spectral gaps of \cref{rem:frames}.
\item The orbit distance~\cref{eq:orbitdist} between the descriptors of two independently processed copies of the same mesh, one rotated by a random $R\in\SO(3)$ and reparameterized by a random M\"obius transformation.  The exact value is zero by \cref{cor:orbit}, so the computed value measures the numerical error of the whole pipeline on two congruent discretizations, free of the nuisance rotations.
\end{enumerate}
Diagnostic (v) is the end-to-end test; if it fails, (i)--(iv) locate the error.

\label{sec:comp-verified}
The following were checked independently of any mesh, by exact rational arithmetic where the inputs are rational and to machine precision otherwise, with the script provided as Online Resource~1; random inputs were drawn with fixed seeds, and an implementation should reproduce these checks before it is run on data.  The conventions of \cref{sec:binary}: orthonormality of the $Y_{lm}$ through $l=5$, the identities $q(v(z))=0$, $\disc Q_x=4q(x)$ and $v(jz)=-\overline{v(z)}$, orthogonality and determinant one of $\rho(g)$ on $ad-bc=1$, reality of $\rho(g)$ with $m(gz)=\rho(g)m(z)$ for $g\in\SU(2)$, and $a_{2l-k}=(-1)^{l+k}\bar a_k$ for random real harmonics of degrees $1,2,3$; an implementation with a different sign in~\cref{eq:veronese} or a different basis phase fails these.  The formulas of \cref{thm:correspondence}(iii) and the norm identity (iv), exactly for $l\le5$ on random integer coefficients, with $J_2(F_{f_3})=2100\norm{f_3}^2$.  The parity of \cref{lem:reality}: on $\lambda Q_aQ_bQ_c$ with rational $a,b,c,\lambda$, $J_2,\dots,J_{10}$ are rational and $R$ is a rational multiple of $i$, and $P_{123}=((F_1,F_2)_1,(F_3,F_3)_4)_4$ is purely imaginary and nonzero.  The identities of \cref{ex:quartic,thm:L2}: $I(F_f)=2\tr Q^2$, $J(F_f)=-4\det Q$, $(Q_a^2,F_f)_4=4a^{\top}Qa$ and $\varepsilon^2=\det\Gram(a,Qa,Q^2a)$ for random rational $a$ and traceless symmetric $Q$.  For \cref{thm:chirality3}: $R=0$ on coplanar and bisector configurations and $R\ne0$ on a generic one, $R\mapsto-R$ under $a\mapsto-a$, $R$ scales by $2^{15}$ when one axis is doubled, on sample curves of the parametrization of $\mathcal L_2$ in~\cite{2000-2} Igusa's invariants agree with the expressions in terms of the roots in~\cite{2024-03} and the forms of weighted degree $30$ vanishing there span the line generated by $F_2$, and $R^2/F_2=-2^{18}3^{-18}5^{-20}$ exactly on random integer axis triples.  On $V_2$, for random rational $a$ and $Q$, the bispectral invariants~\cref{eq:bispectral} are $4a^{\top}Qa$ for the band triples $(1,1,2)$, $(1,2,1)$, $(2,1,1)$ and $-24\det Q$ for $(2,2,2)$, and the others, including $B_{122}$, vanish identically.  Finally, the conformal barycenter of a measure of $42$ equal point masses on $\Sph$ is found by Newton's method with residual below $10^{-16}$.  The groups of \cref{exa:n7} were computed by a separate script (Online Resource~4): the rotation group of the axis set in exact rational arithmetic, and the reduced automorphism groups by enumerating M\"obius transformations with $60$-digit arithmetic and checking that they permute the roots to $10^{-30}$.

\subsection{A synthetic end-to-end test}
\label{sec:comp-synthetic}

Diagnostic (v) was run, with the script provided as Online Resource~2, on the graph $u\mapsto r(u)\,u$ over $\Sph$ with
\begin{multline*}
r=1+\tfrac14x_3+\tfrac7{20}\bigl(x_1^2-\tfrac12x_2^2\bigr)+\tfrac14x_1x_3+\tfrac15x_1x_2x_3\\
+\tfrac3{20}\bigl(x_3^3-\tfrac35x_3\bigr)+\tfrac3{25}\bigl(x_1^3-3x_1x_2^2\bigr)x_3-\tfrac1{10}x_2x_3^2,
\end{multline*}
meshed on the icosahedral subdivision with $n=10242$ vertices and $20480$ faces.  It is chiral and has no continuous symmetry.  Copy A is this mesh.  Copy B is the mesh rotated by a random $R\in\SO(3)$, with its vertices relabeled by a random permutation and its initial spherical map composed with a random M\"obius transformation of translation length $0.45$.  Each copy was processed independently: the conformal map by minimization of the harmonic energy with cotangent weights, the barycenter by Newton's method, the coefficients through $L=6$ with round weights, and the invariants of \cref{thm:L2} and of band three by the code of Online Resource~1.  \Cref{tab:synthetic} lists the diagnostics.

\begin{table*}[t]
\caption{Diagnostics of \cref{sec:comp-diag} for two independently processed copies of one synthetic mesh.  Copy B agrees with copy A to the digits shown.  The last row is a first-order sensitivity ratio; see the text.}
\label{tab:synthetic}
\small
\begin{tabular}{@{}p{0.64\textwidth}p{0.33\textwidth}@{}}
\toprule
diagnostic & value (copy A) \\
\midrule
(i) maximal weight $\max_iw_i$; centering residual $\abs{\sum_iw_iu_i}$ & $1.9\cdot10^{-4}$; $6\cdot10^{-16}$ \\
\phantom{(i)} quasi-conformal distortion of the discrete map, max / mean (additional) & $1.018$ / $1.008$ \\
(ii) $\norm{\mathsf Y^{\top}\Omega\mathsf Y-1}_2$; $\norm{C_0}$ & $1.4\cdot10^{-3}$; $2.3\cdot10^{-5}$ \\
(iii) truncation residual at $L=6$; bound of \cref{prop:truncation}; ratio & $2.5\cdot10^{-5}$; $4.5\cdot10^{-2}$; $5.6\cdot10^{-4}$ \\
\phantom{(iii)} $\sup e^{2\rho}$ & $3.10$ \\
(iv) quadrupole eigenvalues & $-0.2186,\ -0.0843,\ 0.3029$ \\
(v) orbit distance $\delta(c,c')/\norm c$, $L=3$ / $L=6$ & $1.8\cdot10^{-8}$ / $2.0\cdot10^{-8}$ \\
\phantom{(v)} seven invariants of \cref{thm:L2}, relative difference A vs.\ B & $10^{-8}$ to $8\cdot10^{-7}$ \\
\phantom{(v)} $J_2,J_4,J_6,J_{10}$, relative difference & $2\cdot10^{-7}$ to $2\cdot10^{-6}$ \\
\phantom{(v)} $-i\tilde R$, relative difference & $2.5\cdot10^{-5}$ \\
\phantom{(v)} $\abs{I_k(c)-I_k(c')}/(\abs{\nabla I_k(c)}\,\delta(c,c'))$, $k=1,\dots,7$ & all $\le0.83$ \\
\botrule
\end{tabular}
\end{table*}

The exact orbit distance between the two copies is zero.  The cotangent weights are intrinsic, so the discrete energy and the sphere constraint are invariant under rigid motions and relabelings; when the discrete minimizer is unique up to a rotation of $\Sph$, the two copies share it, the discretization error is common to both, and the computed distance $2\cdot10^{-8}$ is the optimization error of the harmonic map and of the minimization in~\cref{eq:orbitdist}.  The invariants agree to the same order.  To first order in $\delta$ the ratios in the last row equal $\abs{\cos\angle(\nabla I_k(c),c-R\cdot c')}$, so they are at most $1+O(\delta)$ by Cauchy--Schwarz; their being below one is consistent with, but not a test of, \cref{thm:stability}(ii), which uses $\sup_B\abs{\nabla I_k}$.  The mirror image of the mesh, processed independently, reproduces the six $\Or(3)$-invariants of \cref{thm:L2} and $J_2,\dots,J_{10}$ to relative $10^{-6}$ and reverses the signs of $\varepsilon$ and $-i\tilde R$, as \cref{thm:certified}(d) requires.  The same surface meshed with $n=2562$ vertices has orbit distance $\delta/\norm c=1.1\cdot10^{-4}$ from copy A, and its invariants differ by relative $3\cdot10^{-4}$ to $3\cdot10^{-3}$ in $J_2,\dots,J_{10}$ and by up to $2\cdot10^{-2}$ in the dipole invariants, whose magnitude here is $\abs a^2=4.7\cdot10^{-4}$.  The discretization error thus exceeds the error of the quotient by four orders of magnitude; it is the quantity a mesh-level validation must control.  Such a validation of the same front end, the conformal map with barycenter centering followed by the harmonic fit, is reported in~\cite{sh-182}: on $26$ structures of the BodyParts3D atlas, voxelized at $1\,$mm after random rotations, reflections, translations and scalings, the median displacement of the invariants is about six percent of the median distance between distinct structures, and the failures are the structures on which the conformal map concentrates area beyond what the mesh samples.

\subsection{Benchmarks on harmonic coefficients}
\label{sec:benchmarks}

\Cref{sec:comp-synthetic} shows that on meshes the discretization error dominates the error of the quotient.  To test the separation and stability statements themselves, the experiments of this section work directly on harmonic coefficients in $V_2$, where the descriptor can be evaluated exactly and the orbit distance computed to high accuracy.  A dipole--quadrupole pair is written $c=(a,Q)$ as in \cref{thm:L2}, with the constant band omitted; the invariants compared are those of \cref{sec:bands}.  Rotations act by $(a,Q)\mapsto(Ra,RQR^{\top})$, and all norms are the $L^2(d\omega)$ norms of the harmonics, $\norm{f_1}^2=\abs a^2/3$ and $\norm{f_2}^2=\tfrac2{15}\tr Q^2$.  The scripts are Online Resource~3.  \Cref{fig:benchmarks} shows the results.

\begin{figure*}[t]
\centering
\resizebox{\textwidth}{!}{
\begin{tikzpicture}
\begin{groupplot}[group style={group size=3 by 1,horizontal sep=1.45cm},width=4.6cm,height=4.4cm,tick label style={font=\scriptsize},label style={font=\scriptsize},title style={font=\scriptsize},legend style={font=\scriptsize,draw=none,fill=none,cells={anchor=west}},every axis plot/.append style={line width=0.7pt,mark size=1.4pt}]
\nextgroupplot[title={(a) four classes $A,A^*,B,B^*$},xlabel={noise $\sigma$ (\% of $\|c\|$)},ylabel={1-NN test accuracy},ymin=0,ymax=1.05,legend columns=2,legend style={at={(0.5,-0.36)},anchor=north},legend style={/tikz/every even column/.append style={column sep=0.2cm}}]
\addplot[color=gray,mark=*,error bars/.cd,y dir=both,y explicit] coordinates {(0,0.25675) +- (0,0.00319668) (0.5,0.255125) +- (0,0.0102347) (1,0.25675) +- (0,0.00988528) (2,0.2565) +- (0,0.00789185) (5,0.2515) +- (0,0.00451733) (10,0.25525) +- (0,0.0055)}; \addlegendentry{power spectrum}
\addplot[color=darkgray,mark=square*,error bars/.cd,y dir=both,y explicit] coordinates {(0,0.268375) +- (0,0.00461993) (0.5,0.2575) +- (0,0.0172934) (1,0.252875) +- (0,0.0190337) (2,0.25325) +- (0,0.0176706) (5,0.257) +- (0,0.0189712) (10,0.2595) +- (0,0.0174795)}; \addlegendentry{bispectrum}
\addplot[color=orange,mark=triangle*,error bars/.cd,y dir=both,y explicit] coordinates {(0,0.494875) +- (0,0.00872675) (0.5,0.50525) +- (0,0.015745) (1,0.506625) +- (0,0.0199711) (2,0.493875) +- (0,0.018462) (5,0.38125) +- (0,0.0184052) (10,0.294625) +- (0,0.0175517)}; \addlegendentry{$\mathrm{O}(3)$-invariants}
\addplot[color=blue,mark=diamond*,error bars/.cd,y dir=both,y explicit] coordinates {(0,1) +- (0,0) (0.5,0.993875) +- (0,0.00228446) (1,0.995875) +- (0,0.00075) (2,0.971125) +- (0,0.00437321) (5,0.736625) +- (0,0.0148039) (10,0.56475) +- (0,0.00996557)}; \addlegendentry{Thm.~3.3}
\addplot[color=green!60!black,mark=triangle*,mark options={rotate=180},error bars/.cd,y dir=both,y explicit] coordinates {(0,1) +- (0,0) (0.5,1) +- (0,0) (1,1) +- (0,0) (2,0.991375) +- (0,0.00225) (5,0.784375) +- (0,0.0131814) (10,0.592375) +- (0,0.0165284)}; \addlegendentry{frame, Rem.~3.4}
\addplot[black,dotted,thin,forget plot] coordinates {(0,0.25) (10,0.25)};
\nextgroupplot[xmode=log,ymode=log,title={(b) conditioning near an axial stratum},xlabel={eigenvalue gap $g=\lambda_1-\lambda_2$},ylabel={local ratio $|\Delta D|/\delta$},legend style={at={(0.5,-0.36)},anchor=north}]
\addplot[color=green!60!black,mark=triangle*,mark options={rotate=180}] coordinates {(0.0001,6294.46) (0.000177828,3495.66) (0.000316228,2024.9) (0.000562341,1159.11) (0.001,625.641) (0.00177828,359.391) (0.00316228,205.724) (0.00562341,111.722) (0.01,63.797) (0.0177828,35.6531) (0.0316228,20.2897) (0.0562341,10.6569) (0.1,6.5498) (0.177828,3.7214) (0.316228,2.72148) (0.562341,2.67179) (1,2.67789)}; \addlegendentry{frame coordinates}
\addplot[color=blue,mark=diamond*] coordinates {(0.0001,4.51235) (0.000177828,4.1865) (0.000316228,4.26008) (0.000562341,4.51963) (0.001,4.5058) (0.00177828,4.56198) (0.00316228,4.47775) (0.00562341,4.52814) (0.01,4.5054) (0.0177828,4.52744) (0.0316228,4.30771) (0.0562341,4.36363) (0.1,4.4239) (0.177828,4.5503) (0.316228,4.4821) (0.562341,4.9156) (1,5.65867)}; \addlegendentry{polynomial invariants}
\addplot[black,dotted,domain=1e-4:1,samples=2] {0.65/x}; \addlegendentry{$\propto1/g$}
\nextgroupplot[xmode=log,ymode=log,title={(c) stability exponents},xlabel={orbit distance $\delta(c,c^\prime)$},ylabel={$\|\mathcal I(c)-\mathcal I(c^\prime)\|$},legend style={at={(0.5,-0.36)},anchor=north}]
\addplot[color=blue,mark=diamond*] coordinates {(0.000316228,0.000833579) (0.000562341,0.00177867) (0.001,0.00210619) (0.00177828,0.00468851) (0.00316228,0.00899313) (0.00562341,0.0124214) (0.01,0.0260554) (0.0177828,0.0644468) (0.0316228,0.0676189) (0.0562341,0.132945) (0.1,0.152305) (0.177828,0.37695) (0.316228,0.777401)}; \addlegendentry{principal point: slope 0.96}
\addplot[color=red,mark=*] coordinates {(0.000316228,6.37956e-07) (0.000562341,1.548e-06) (0.001,5.32973e-06) (0.00177828,1.44311e-05) (0.00316228,4.38964e-05) (0.00562341,0.0001843) (0.01,0.000363626) (0.0177828,0.0020508) (0.0316228,0.0064225) (0.0562341,0.0116171) (0.1,0.0405098) (0.177828,0.146312) (0.316228,0.494601)}; \addlegendentry{axial point: slope 1.97}
\addplot[color=darkgray,mark=square*] coordinates {(0.000316228,4.11517e-07) (0.000562341,1.4815e-06) (0.001,5.52592e-06) (0.00177828,1.87528e-05) (0.00316228,5.03122e-05) (0.00562341,0.000118863) (0.01,0.000465719) (0.0177828,0.00133978) (0.0316228,0.00502736) (0.0562341,0.0122005) (0.1,0.0374079) (0.177828,0.18578) (0.316228,0.503961)}; \addlegendentry{origin: slope 2.00}
\end{groupplot}
\end{tikzpicture}}
\caption{Benchmarks of \cref{sec:benchmarks}.  (a) Nearest-neighbour accuracy against coefficient noise.  (b) Local Lipschitz ratio as an eigenvalue gap of the quadrupole tends to zero.  (c) Change of the invariants of \cref{thm:L2} against the orbit distance.}
\label{fig:benchmarks}
\end{figure*}

\label{sec:bench-disc}

Let $Q_0=\diag(\lambda)$ with $\lambda=(0.7,0.1,-0.8)$, and put $d=(\lambda_2-\lambda_3,\lambda_3-\lambda_1,\lambda_1-\lambda_2)$, so that $\sum_id_i=\sum_i\lambda_id_i=0$.  Let $a_A$ and $a_B$ have components $\sqrt{w_i}$ with $w=(0.10,0.20,0.10)$ and $w+0.06\,d$, respectively.  The pairs $A=(a_A,Q_0)$ and $B=(a_B,Q_0)$ then have the same $\abs a^2$, $\tr Q^2$, $\det Q$ and $a^{\top}Qa$, but different $a^{\top}Q^2a$ ($0.115$ against $0.164$).  Their mirror images are $A^*=(-a_A,Q_0)$ and $B^*=(-a_B,Q_0)$ by \cref{lem:reality}(iv); they have the same $\Or(3)$-invariants as $A$ and $B$ and the opposite sign of $\varepsilon$.  Each sample is a class representative rotated by a uniformly random $R\in\SO(3)$, with independent Gaussian noise of standard deviation $\sigma\norm c$ added to each coefficient in the orthonormal basis $\{Y_{lm}\}$.

Five descriptors were compared: the power spectrum $(\norm{f_1},\norm{f_2})$; the bispectrum, that is, the invariants~\cref{eq:bispectral} with bands $l\le2$; the five nonconstant $\Or(3)$-invariants of \cref{thm:L2}; all seven invariants of \cref{thm:L2}; and the frame coordinates of \cref{rem:frames}, the $K$-invariant monomials $\lambda_1,\lambda_2,\tilde a_1^2,\tilde a_2^2,\tilde a_3^2,\tilde a_1\tilde a_2\tilde a_3$ in the eigenframe-aligned dipole $\tilde a$.  On $V_2$ the bispectral invariants are, up to the constant band, $4a^{\top}Qa$ for the band triples $(1,1,2)$, $(1,2,1)$, $(2,1,1)$ and $-24\det Q$ for $(2,2,2)$, and all others vanish, including $B_{122}$, the only candidate pseudo-invariant (\cref{sec:comp-verified}).  Each feature $x_k$ of degree $w_k$ was replaced by $\sgn(x_k)\abs{x_k}^{1/w_k}$ (see \cref{sec:scale}) and standardized on the training set, and each class was represented by $100$ training and $400$ test samples in a nearest-neighbour classifier.

\Cref{fig:benchmarks}(a) shows the outcome as the mean and standard deviation over five repetitions; chance is $1/4$.  The power spectrum and the bispectrum stay at chance at every noise level, including $\sigma=0$, because they are functions of invariants that coincide on all four classes.  The $\Or(3)$-invariants reach one half: they separate $A$ from $B$ but not a class from its mirror image, as \cref{prop:chirality} predicts.  The seven invariants of \cref{thm:L2} and the frame coordinates classify all four classes without error at $\sigma=0$ and with accuracy above $0.97$ at $\sigma=0.02$.  For larger noise all descriptors degrade, and the ordering is unchanged.  The per-coefficient level $\sigma=0.05$ corresponds to a noise vector of about $14\%$ of $\norm c$, comparable to the separation $\abs{\varepsilon}\approx0.037$ between mirror classes.

\label{sec:bench-frames}
Let $Q_g=\diag(\tfrac13+\tfrac g2,\tfrac13-\tfrac g2,-\tfrac23)$ and $a=(0.3,0.4,0.5)$, so that $c_g=(a,Q_g)$ tends to the axial stratum $\lambda_1=\lambda_2$ of \cref{rem:frames} as $g\to0$.  For each $g$, $300$ perturbations $c_g+\eta u$ with $\eta=10^{-7}$ and $u$ a random unit vector normal to the orbit were drawn; to first order $\delta(c_g,c_g+\eta u)=\eta$.  \Cref{fig:benchmarks}(b) shows the largest ratio $\abs{D(c_g+\eta u)-D(c_g)}/\eta$ for the frame coordinates and for the invariants of \cref{thm:L2}.  For the frame the ratio grows like $0.65/g$, with fitted slope $-1.00$ over $10^{-4}\le g\le10^{-2}$, which is the conditioning and the discontinuity described in \cref{rem:frames}.  For the polynomial invariants the ratio stays between $4.2$ and $5.7$ over four decades of $g$, in accordance with \cref{thm:stability}(ii).  This is the reason for preferring the polynomial invariants on data whose spectral gaps are not controlled.

\label{sec:bench-exponents}
At three base points $c_0$, perturbations $c=c_0+\eta u$ with $3\cdot10^{-4}\le\eta\le0.3$ and $u$ normal to the orbit were drawn, and $\delta(c_0,c)$ was computed by multistart minimization over $\SO(3)$; at the smallest $\eta$ it agreed with $\eta$ to four digits.  The base points were the principal point $A$; the point $(0,0,\tfrac12)$, $\diag(\tfrac12,\tfrac12,-1)$, whose stabilizer is the group $\SO(2)$ of rotations about the $x_3$-axis, perturbed in the directions of its slice that this group moves; and the origin.  \Cref{fig:benchmarks}(c) shows the median of $\norm{\mathcal I(c)-\mathcal I(c_0)}$ against $\delta$, with $\mathcal I$ the seven invariants of \cref{thm:L2}.  The fitted slopes are $0.96$, $1.97$ and $2.00$.  At the principal point the invariants are locally bi-Lipschitz, as \cref{thm:stability}(iv) asserts.  At the other two points the invariant discrepancy is quadratic in the orbit distance, so the exponent $\alpha_{\mathcal K}$ of \cref{thm:stability}(iii) cannot exceed $1/2$ on a neighborhood of them, as \cref{exa:holder} asserts.

The experiments give a consistent picture.  On $V_2$ the power spectrum and the bispectrum are stable but incomplete and blind to reflections.  Frames are complete where $Q$ has distinct eigenvalues but ill-conditioned near the axial strata.  The polynomial invariants are complete, record chirality, and are Lipschitz on bounded sets; their inverse is locally Lipschitz at the points of \cref{prop:exponents}(i) and only H\"older, with exponent at most $1/2$, at the strata where the stabilizer jumps.

\section{Closing Remarks}
\label{sec:discussion}

For a real spherical harmonic of odd degree $n\ge3$ with distinct Maxwell axes, chirality is governed by the automorphisms of the hyperelliptic curve of genus $n-1$ attached to its binary form: the harmonic is congruent to its mirror image iff the curve has a cyclic reduced automorphism of even order whose lift acts on the form by $-1$, and the locus of such curves is the union of the locus of curves with an extra involution and of the cyclic strata of \cref{thm:strata}.  The reduced automorphism groups of hyperelliptic curves and the loci they define are completely known, see~\cite{2003-3}, and the theorem places the achiral harmonics among them; the real harmonics are singled out by the antipodal real structure of their curves (\cref{thm:real}).  In degree three the curves have genus two, the achiral locus is the Humbert surface $\mathcal L_2$ of curves with a degree-two elliptic subcover, and the criterion is the vanishing of the skew invariant, by Clebsch; the loci $\mathcal L_m$ of genus-two curves with a degree-$m$ elliptic subcover for $m\ge3$, that is, the Humbert surfaces of discriminant $m^2$, and the configurations of Maxwell axes of the real harmonics of degree three whose sextic lies on them, are the subject of~\cite{ShaskaHumbert}.

The invariants that decide chirality are the pseudo-invariants, and for a single band of odd degree these are the invariants of odd degree of the binary form.  Explicit generators of the invariant ring, and hence an equation of the achiral locus in the invariants such as~\cref{eq:R2F2}, are available only in low degree; the strata of \cref{thm:strata} are instead parametrized by the coefficients of $K$ modulo the normalizer of the cyclic group, in the manner of the dihedral invariants of~\cite{2003-3}.  For a tuple of several bands, the separating sets available are the frame coordinates of \cref{rem:frames} on the set $V_L^\circ$ where $Q$ has distinct eigenvalues, monomials in $\dim V_L-3$ aligned coordinates, and the $2\dim V_L+1$ generic linear combinations of invariants of~\cite{DymGortler24} globally; the fibers of the standard incomplete pools of invariants are computed in~\cite{sh-183}.

The stability of the invariants in the orbit distance is Lipschitz in one direction and H\"older in the other, with exponent $1$ at the points of the principal stratum where the invariants have full rank on a slice, and at most $1/2$, $1/3$ and $1/6$ at the strata of $V_2$ where the stabilizer jumps (\cref{prop:exponents}); the bound at the origin is forced by chirality, since the pseudo-invariant that separates a tuple from its mirror image has degree six.  On surfaces, the conformal parameterization normalized by its barycenter makes these invariants a descriptor that is exactly independent of parameterization, position and scale (\cref{thm:certified}), and the computation of \cref{sec:computation} shows that on meshes the discretization error, not the error of the quotient, is the quantity to control.

\subsection*{Supplementary material}
The verification script of \cref{sec:comp-verified} (Online Resource~1), the script of the synthetic test of \cref{sec:comp-synthetic} (Online Resource~2), the scripts of the benchmarks of \cref{sec:benchmarks} (Online Resource~3) and the script of \cref{exa:n7} (Online Resource~4) are supplied as electronic supplementary material.

\section*{Declarations}

\subsection*{Funding}
The author did not receive support from any organization for the submitted work.

\subsection*{Competing interests}
The author has no competing interests to declare that are relevant to the content of this article.

\subsection*{Data availability}
No datasets were generated or analysed in this work.  The scripts used for the computations of \cref{sec:comp-verified,sec:comp-synthetic,sec:benchmarks} and of \cref{exa:n7} are provided as Online Resources~1--4.

\providecommand{\bysame}{\leavevmode\hbox to3em{\hrulefill}\thinspace}
\providecommand{\MR}{\relax\ifhmode\unskip\space\fi MR }
\providecommand{\MRhref}[2]{%
  \href{http://www.ams.org/mathscinet-getitem?mr=#1}{#2}
}
\providecommand{\href}[2]{#2}

\end{document}